\documentclass[a4paper,fleqn]{cas-dc}

\usepackage{amsmath,amssymb,amsfonts}
\usepackage{amsthm}
\usepackage{graphicx}
\usepackage[numbers]{natbib}
\usepackage{booktabs}
\usepackage{multirow}
\usepackage{bbm}
\usepackage{bm}
\usepackage{mathtools}
\usepackage{algorithm}
\usepackage{algpseudocode}
\usepackage{array}
\usepackage{hyperref}
\usepackage{xcolor}
\usepackage{balance}
\newtheorem{definition}{Definition}
\newtheorem{theorem}{Theorem}

\newtheorem{corollary}{Corollary}

\newcommand{\cU}{\mathcal{U}}
\newcommand{\cH}{\mathcal{H}}

\newcommand{\cO}{\mathcal{O}}
\newcommand{\cY}{\mathbf{Y}}
\newcommand{\cG}{\mathbf{G}}
\newcommand{\R}{\mathbb{R}}
\newcommand{\C}{\mathbb{C}}

\newcommand{\norm}[1]{\left\lVert#1\right\rVert}

\newcommand{\simU}{\sim_{\cU,\mathcal{C}}}
\newcommand{\Qscene}{\mathcal{Q}_{\scene}^{\cU,\mathcal{C}}}
\newcommand{\thetaspace}{\Theta_{\mathrm{phys}}}

\newcommand{\obs}{\mathrm{obs}}
\newcommand{\phys}{\mathrm{phys}}
\newcommand{\res}{\mathrm{res}}
\newcommand{\eq}{\mathrm{eq}}
\newcommand{\dist}{\mathrm{dist}}
\newcommand{\amb}{\mathrm{amb}}
\newcommand{\act}{\mathrm{act}}
\newcommand{\nui}{\mathrm{nui}}
\newcommand{\scene}{\mathrm{scene}}
\newcommand{\rx}{\mathrm{rx}}
\newcommand{\calib}{\mathrm{cal}}
\newcommand{\offset}{\mathrm{offset}}
\newcommand{\transient}{\mathrm{transient}}

\newcommand{\inv}{\mathrm{inv}}
\newcommand{\can}{\mathrm{can}}
\newcommand{\raw}{\mathrm{raw}}
\newcommand{\calop}{\mathcal{C}}
\newcommand{\calN}{\mathcal{N}}

\def\tsc#1{\csdef{#1}{\textsc{\lowercase{#1}}\xspace}}
\tsc{WGM}
\tsc{QE}
\tsc{EP}
\tsc{PMS}
\tsc{BEC}
\tsc{DE}

\begin{document}
\let\WriteBookmarks\relax
\def\floatpagepagefraction{1}
\def\textpagefraction{.001}
\shorttitle{Sensor-Conditioned Representation Learning via Scene-Relevant Observation Quotients}
\shortauthors{Y. Jiao et al.}

\title [mode = title]{Sensor-Conditioned Representation Learning via Scene-Relevant Observation Quotients}         
\author[1]{Yan Jiao}[orcid=0009-0000-9028-4614]
\ead{yanyanjiao2018@gmail.com}
\author[1,2]{Pin-Han Ho}[orcid=0000-0002-0717-1481]
\cormark[1]
\ead{pinhanho71@gmail.com}
\author[1,3]{Limei Peng}[orcid=0000-0001-9984-9861]
\ead{auroraplm@knu.ac.kr}
\affiliation[1]{organization={Shenzhen Institute for Advanced Study, 
University of Electronic Science and Technology of China},
                city={Shenzhen},
                country={China}}
\affiliation[2]{organization={Department of Electrical and Computer Engineering, 
University of Waterloo},
                city={Waterloo},
                country={Canada}}
\affiliation[3]{organization={School of Computer Science and Engineering, 
Kyungpook National University},
                city={Daegu},
                country={South Korea}}
\cortext[1]{Corresponding author}

\begin{abstract}
Learned representations in intelligent sensing systems are commonly evaluated by reconstruction fidelity or downstream prediction accuracy. However, these criteria do not determine whether a latent representation preserves the distinctions that the sensing process can meaningfully support. In sensor-conditioned environments, nuisance factors may induce measurement variability without scene changes, while physically different scenes may remain indistinguishable under limited sensing capability. This raises a fundamental question: \emph{what distinctions should a sensor-conditioned representation preserve?}

This paper formulates \emph{sensor-conditioned representation correctness} as preserving sensing-supported scene distinctions while suppressing nuisance-induced and sensor-unsupported variations. We introduce the \emph{scene-relevant observation quotient}, a representation target induced by sensing-supported distinguishability after nuisance canonicalization. Based on this formulation, we develop Observation-Quotient Tucker-Structured Autoencoding (OQ-TSAE), a scene--nuisance factorized framework for quotient-consistent representation learning, together with diagnostics for false distinction, false merge, nuisance sensitivity, and latent ordering consistency.

Experiments on a controlled sensor-conditioned benchmark show that quotient-consistent supervision improves representation-correctness diagnostics over reconstruction-oriented, metric-learning, and contrastive-learning baselines. Sensitivity, perturbation, and ablation studies further demonstrate the importance of quotient-aligned supervision, reliable quotient relations, and quotient geometry. Complementary real-radar experiments show that a reconstruction-only OQ-TSAE variant retains competitive downstream utility, favorable robustness under observation degradation, and low seed-to-seed variability. These results suggest that sensor-conditioned representations should be evaluated not only by predictive utility, but also by whether their latent geometry preserves sensing-justified scene distinctions.

\end{abstract}

\begin{highlights}
\item Studies representation correctness under structured sensing.
\item Defines scene-relevant observation quotients after canonicalization.
\item Connects quotient consistency to representation correctness.
\item Proposes OQ-TSAE for quotient-consistent representation learning.
\item Shows controlled quotient correctness and complementary real-radar utility.
\end{highlights}

\begin{keywords}
    Representation Learning \sep
    Observation Quotient \sep
    Representation Correctness \sep
    Structured Sensing \sep
    Scene-Relevant Distinguishability \sep
    Sensor-Conditioned Representation Learning
\end{keywords}
\maketitle

\section{Introduction}
\label{sec:introduction}

\subsection{Motivation: Representation Correctness Under Sensing Constraints}
\label{subsec:motivation_representation_correctness_under_sensing_constraints}

Learned representations in intelligent sensing systems are commonly evaluated by reconstruction fidelity, downstream prediction accuracy, detection performance, or localization error. These criteria are useful, but they do not directly determine whether the latent representation preserves the distinctions that the sensing process can meaningfully support. In sensor-conditioned environments, the same physical scene may produce different raw measurements because of receiver gain drift, calibration phase offsets, hardware state, interference, or environmental uncertainty. Conversely, physically different scenes may remain indistinguishable when the sensing system lacks sufficient resolution, aperture diversity, action diversity, or signal-to-noise reliability. Thus, raw measurement variability is not automatically meaningful scene variability, and physical scene variability is not automatically observable variability.

This issue is especially prominent in structured sensing systems. Multi-input multi-output (MIMO) radar, frequency-diverse apertures, and reconfigurable sensing platforms generate observations through a coupled interaction among the scene, sensing configuration, hardware response, calibration state, and environmental uncertainty~\cite{li2008mimo,liu2022integrated,tang2020wireless,karl2023foundations}. In such systems, the observation process itself determines which distinctions can be reliably supported. A change in the raw measurement may reflect a nuisance perturbation rather than a scene change, while a real physical change may be below the distinguishability limit of the available sensing-action family. Sensor-conditioned representation learning is therefore constrained not only by the latent scene, but also by the observation process through which the scene is accessed.

Most sensing-learning pipelines do not explicitly model this constraint at the level of representation geometry. A sensing observation is encoded into a latent representation and then used for occupancy estimation, localization, detection, reconstruction, or other downstream tasks~\cite{major2019vehicle,feng2020deep}. The learned representation is usually considered useful if it improves the downstream metric. However, downstream performance alone does not answer a more fundamental question: whether the learned latent space separates and merges observations for the right reasons. A representation may support a particular downstream predictor while still encoding hardware drift as scene structure, or while collapsing scene differences that remain resolvable under the sensing process~\cite{locatello2019challenging,scholkopf2021toward}. Such a representation may be task-useful in one setting but geometrically incorrect as a scene representation.

The central challenge addressed in this paper is therefore \emph{representation correctness}. We use this term to denote whether a learned scene representation preserves distinctions supported by sensing while collapsing distinctions that are nuisance-induced or observationally unsupported. This notion is stricter than downstream predictive accuracy and different from generic compactness, invariance, or reconstruction fidelity. It concerns the validity of the representation itself: which observations should be separated, which should be merged, and what sensing-grounded criterion justifies this organization.

Figure~\ref{fig:intro_overview} illustrates the motivating failure modes together with the representation criterion developed in this paper.
\begin{figure*}
\centering
\includegraphics[width=\textwidth]{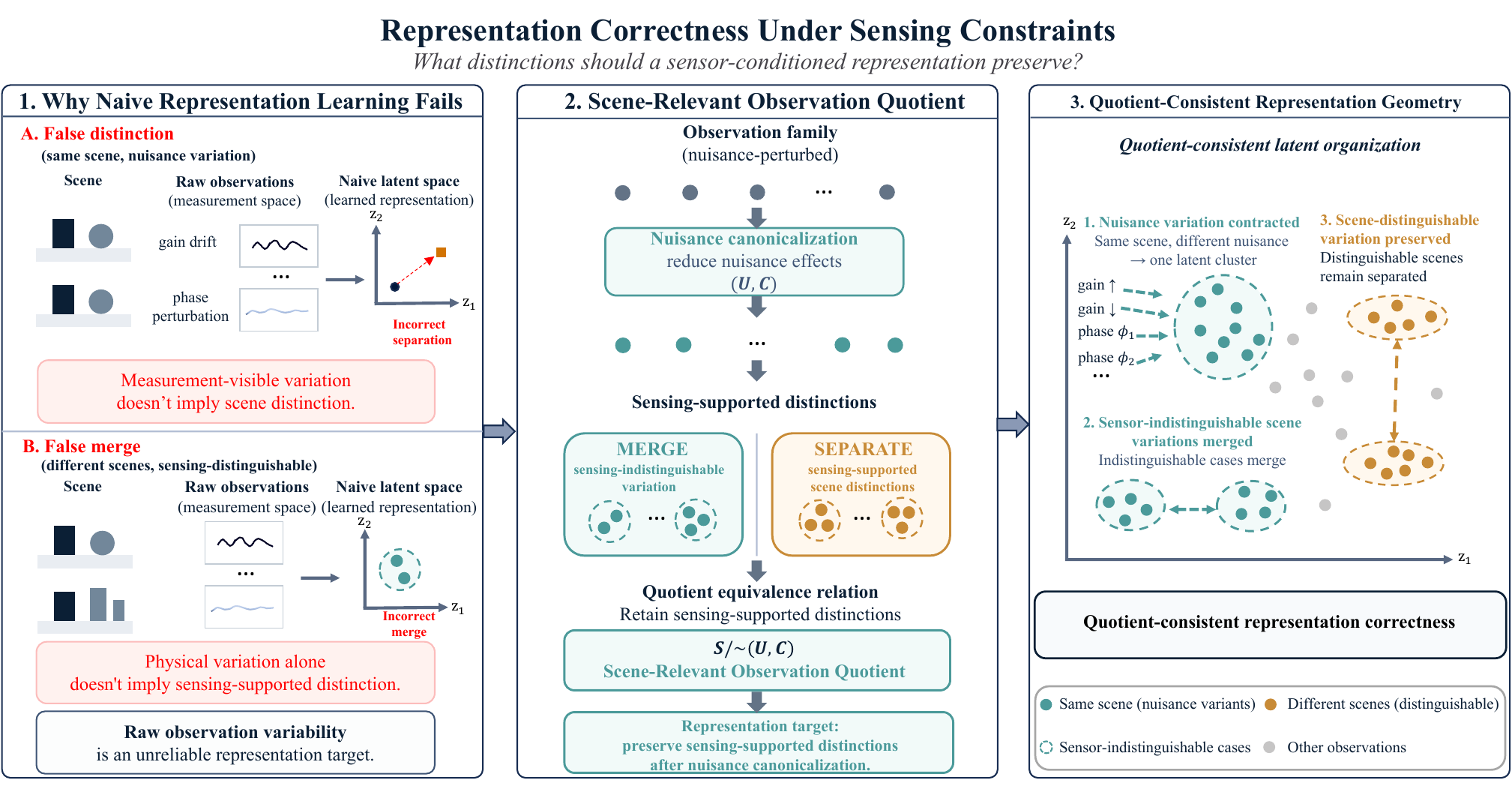}
\caption{
Representation correctness under sensing constraints. The scene-relevant observation quotient defines meaningful representation structure through sensing-supported distinguishability after nuisance canonicalization, avoiding both false distinctions and false merges.
}
\label{fig:intro_overview}
\end{figure*}

The first failure mode is \emph{false distinction}. Observations generated by the same underlying scene may be separated in the latent space because nuisance factors such as receiver gain drift, calibration phase offset, baseline shift, or transient interference alter the raw measurements. A representation driven by raw measurement distinguishability may therefore encode nuisance as scene structure. The second failure mode is \emph{false merge}. Scene configurations that remain distinguishable under the sensing process may be mapped to nearly identical latent representations, thereby discarding scene information that the sensor can in fact support.

These two failures are not merely opposite ends of a simple invariance tradeoff. Preserving all measurement-visible variation is incorrect because nuisance variation can dominate the raw observation. Collapsing all nuisance-associated variation is also insufficient because overly broad invariance can suppress scene differences that remain physically resolvable under the available sensing process. Thus, the goal is neither to preserve everything nor to remove everything except a task label. The goal is to preserve exactly those scene distinctions that are justified by the sensing process.

This leads to the organizing question of this paper:
\emph{what distinctions should a sensor-conditioned representation preserve?}
We answer this question by introducing the \emph{scene-relevant observation quotient}. The quotient defines scene equivalence through sensing-supported distinguishability after nuisance canonicalization. It serves as a representation target that specifies which latent distinctions are justified, which nuisance-induced differences should collapse, and which physical differences are not meaningful because they are unsupported by the available sensing process.

The proposed perspective changes the role of representation learning in intelligent sensing systems. Instead of treating latent representations only as intermediate variables optimized for a decoder or downstream predictor, we treat them as objects whose geometry should be evaluated against a sensing-grounded correctness criterion. A sensor-conditioned representation should preserve scene distinctions that remain supported by sensing after nuisance effects have been canonicalized, while collapsing both nuisance-induced variation and sensor-indistinguishable scene variation. This motivates a representation target defined by sensing-supported distinguishability rather than by raw measurements, generic invariance, or downstream task objectives alone.

\subsection{The Scene-Relevant Observation Quotient}
\label{subsec:the_scene_relevant_observation_quotient}

We formalize this representation target through the \emph{scene-relevant observation quotient}. The starting point is the distinguishability structure induced by the sensing process. Observational equivalence~\cite{hermann2003nonlinear,lennart1999system} provides a natural formalization of this idea: two scene configurations should not be arbitrarily separated if the sensing process cannot reliably distinguish them. Conversely, configurations that remain distinguishable under sensing should not be collapsed. This perspective shifts representation learning from preserving raw observation geometry to preserving the sensing-supported distinguishability structure induced by the observation process.

However, raw observational equivalence alone is not sufficient for scene representation learning because raw measurements jointly depend on scene variables and nuisance variables. Two observations corresponding to the same underlying scene may become distinguishable because of receiver gain drift, calibration phase offset, baseline variation, or transient interference. If equivalence is defined directly in raw measurement space, such nuisance effects may incorrectly become part of the representation target. The resulting latent space may be faithful to measurement variability but incorrect as a scene representation.

We therefore define equivalence after nuisance canonicalization rather than directly in raw measurement space. The canonicalization step suppresses known or estimable nuisance degrees of freedom while preserving scene-induced measurement structure. Scene configurations are treated as equivalent when their canonicalized sensing responses remain indistinguishable under the available sensing family. Conversely, scene configurations are separated when their differences remain supported by sensing after canonicalization. As a result, nuisance-induced variation of the same scene collapses,
physically different scenes that remain indistinguishable to the sensing system collapse,
and physically different scenes that remain sensor-distinguishable are preserved as separate quotient classes.

This quotient formulation clarifies the distinction between the proposed approach and several adjacent representation learning paradigms. Unlike invariant representation learning, the objective is not to remove all variation associated with a nuisance label, domain shift, or augmentation family. The objective is to remove nuisance-induced variation only insofar as it does not correspond to sensing-supported scene distinction. Unlike disentangled representation learning, the goal is not factor independence as an end in itself, but quotient-consistent scene organization. Unlike contrastive learning, the positive and negative relationships are not defined only by sample augmentation or label construction, but by sensing-supported distinguishability after nuisance canonicalization. Unlike reconstruction-oriented representation learning, preserving measurement-visible variation is not sufficient, because some measurement-visible variation is scene-irrelevant.

The key point is that nuisance canonicalization is not merely a preprocessing step for improved reconstruction. It defines the distinguishability relation that determines which latent distinctions are justified. The scene-relevant observation quotient therefore acts as a representation-correctness target: a learned scene latent should contract nuisance-equivalent and sensor-indistinguishable cases, while preserving scene distinctions that the sensing process can meaningfully support.

This perspective naturally motivates representation-level diagnostics beyond downstream prediction accuracy. A representation should be evaluated by whether it avoids false distinctions, avoids false merges, suppresses nuisance sensitivity, and organizes latent geometry according to scene-relevant distinguishability. These diagnostics are central in the experimental evaluation of this paper.

\subsection{Observation-Quotient Tucker-Structured Autoencoding}
\label{subsec:observation_quotient_tucker_structured_autoencoding}

Building on this formulation, we propose \emph{Observation-Quotient Tucker-Structured Autoencoding} (OQ-TSAE), a Tucker-structured scene--nuisance factorized framework for sensor-conditioned observations. OQ-TSAE is designed to operationalize the scene-relevant observation quotient rather than merely optimize reconstruction fidelity. The framework preserves the multi-way organization of sensing observations through a shared Tucker-structured encoder, while separating scene-relevant and nuisance-dependent variation into dedicated latent pathways.

The scene pathway is constrained by quotient-oriented supervision. Its normalized latent representation is shaped according to sensing-supported distinguishability after nuisance canonicalization. In contrast, the nuisance pathway retains additional measurement-dependent variation needed for reconstructing raw observations, including hardware-dependent and calibration-related effects. This factorization addresses a basic tension in sensor-conditioned representation learning. A single reconstruction-oriented latent code is naturally incentivized to preserve nuisance variation because nuisance affects raw measurements. Conversely, an overly invariant latent code may suppress scene differences that remain distinguishable under the sensing process. OQ-TSAE separates these roles: the scene latent is organized by quotient correctness, while the nuisance latent provides reconstruction compatibility.

The Tucker structure provides an inductive bias aligned with the multilinear organization of sensing observations across temporal, sensing, response, and range-related dimensions. Many structured sensing systems naturally generate observations with explicit mode structure, and flattening these observations can discard useful organization. Nevertheless, tensor structure alone is not the central contribution of this paper. A standard tensor autoencoder can preserve multi-way measurement structure while still encoding nuisance variation as scene structure. Similarly, a reconstruction-oriented tensor model can achieve low reconstruction error without producing a quotient-correct scene representation.

OQ-TSAE should therefore be understood as an implementation of a representation-correctness principle, not as a generic tensor-compression architecture. The Tucker-structured encoder supplies a physically compatible representation mechanism for multi-way observations, but the scene-relevant observation quotient defines the learning target. The primary objective is to align the scene latent with
sensing-supported distinguishability,
contract nuisance variation and sensor-indistinguishable cases,
and preserve sensor-supported scene distinctions. In this sense, OQ-TSAE is one concrete realization of a broader quotient-consistent representation-learning framework for sensor-conditioned intelligent systems.

\subsection{Contributions and Paper Organization}
\label{subsec:contributions_and_paper_organization}

This paper makes four primary contributions.

\textbf{First,} we formulate \emph{sensor-conditioned representation correctness} as the problem of preserving sensing-supported scene distinctions while suppressing nuisance-induced and sensor-unsupported variations. This formulation highlights a limitation of evaluating learned sensing representations solely through downstream prediction or reconstruction accuracy.

\textbf{Second,} we introduce the \emph{scene-relevant observation quotient} as a principled representation target. The proposed quotient defines representation correctness through sensing-supported distinguishability after nuisance canonicalization, rather than through raw measurement variability, physical scene identity, generic invariance, or downstream task performance alone.

\textbf{Third,} we develop a theoretical framework for quotient-consistent representation learning and propose OQ-TSAE as a concrete learning architecture. The analysis establishes when computable scene distinguishability is consistent with observational equivalence, characterizes stability under imperfect nuisance canonicalization, and links quotient-oriented supervision to representation-correctness diagnostics. OQ-TSAE operationalizes the quotient target through a Tucker-structured scene--nuisance factorized autoencoding framework that remains compatible with observation reconstruction, scene decoding, and competitive downstream utility under real-radar evaluation.

\textbf{Fourth,} we introduce quotient-oriented diagnostics for evaluating representation correctness beyond downstream prediction accuracy. The proposed diagnostics assess false distinctions, false merges, nuisance sensitivity, and latent ordering consistency under controlled scene-pair families. The controlled benchmark is used not as a sensor-specific state-of-the-art benchmark, but as a mechanism-level testbed in which quotient structure, nuisance variation, and scene distinguishability relationships can be explicitly evaluated, while complementary real-radar experiments provide additional assessments of the learned scene--nuisance tensor representations in terms of downstream utility, robustness, and stability.

The remainder of the paper is organized as follows. Section~\ref{sec:related_work} positions the work relative to structured sensing, invariant representation learning, observability, and tensor representation models. Section~\ref{sec:sensor_conditioned_scene_state_formulation} formulates the sensor-conditioned observation model, nuisance canonicalization, and the scene-relevant observation quotient. Section~\ref{sec:theoretical_analysis} analyzes quotient consistency, canonicalization stability, and diagnostic consistency. Section~\ref{sec:scene_nuisance_factorized_oq_tsae} introduces OQ-TSAE. Section~\ref{sec:learning_objectives_and_evaluation_metrics} presents the learning objectives and quotient-oriented evaluation diagnostics. Section~\ref{sec:controlled_sensor_conditioned_benchmark} describes the controlled benchmark and scene-pair taxonomy.
Section~\ref{sec:experimental_evaluation} reports experimental evaluation. Section~\ref{sec:discussion} discusses interpretation, implications, and limitations. Section~\ref{sec:conclusion} concludes the paper.

\section{Related Work}
\label{sec:related_work}

The proposed work is adjacent to several mature research directions: learned representations for structured sensing, nuisance-invariant and disentangled representation learning, observability-based distinguishability analysis, and tensor-structured representation models. These areas provide important technical ingredients, but they do not by themselves define the representation target studied in this paper. The central distinction of the present work is that we study sensor-conditioned representation correctness rather than reconstruction fidelity, downstream prediction, nuisance invariance, or tensor compression alone. Specifically, we ask which scene distinctions should remain separated after nuisance canonicalization and which should collapse because they are not reliably supported by the sensing process. This section positions the proposed scene-relevant observation quotient against the closest existing paradigms and clarifies why they do not directly address representation correctness under sensing constraints.

\subsection{Structured Sensing and Learned Sensor Representations}
\label{subsec:sensor_conditioned_representation_learning}

Structured sensing platforms such as frequency-modulated continuous-wave (FMCW) radar, millimeter-wave imaging systems, MIMO radar, and frequency-diverse apertures have enabled increasingly capable perception pipelines for detection, localization, occupancy inference, and scene understanding~\cite{patole2017automotive,schumann2021radarscenes,harlow2024new,borts2024radar}. In these systems, observations are not generic feature vectors. They are complex-valued measurements generated by a physical sensing process, shaped by controllable sensing actions, aperture responses, range resolution, calibration state, propagation effects, and noise. This structure motivates representation models that exploit sensing geometry and multi-way measurement organization rather than treating observations as generic feature vectors.

Most learned sensing pipelines, however, evaluate latent representations through task-level or reconstruction-level criteria. A representation is usually considered effective if it improves occupancy estimation, localization, detection, scene reconstruction, or raw measurement prediction. These criteria are useful, but they do not answer whether the latent geometry preserves the distinctions that the sensing platform can actually support. A model can achieve strong task accuracy while encoding receiver gain drift as scene structure, or while collapsing scene differences that remain resolvable under the available sensing actions. Conversely, a reconstruction-oriented representation may faithfully preserve measurement variability that is irrelevant to the scene.

This limitation motivates treating sensor-conditioned representation learning as a problem of representation correctness rather than only predictive performance. The relevant question is not simply whether a latent code is useful for a downstream decoder, but whether it separates scene configurations that remain distinguishable after nuisance canonicalization and contracts those that are indistinguishable or scene-irrelevant under the sensing process. Existing learned sensing representations are typically not formulated in terms of an explicit quotient-level representation target.

\subsection{Invariant, Nuisance-Aware, and Disentangled Representations}
\label{subsec:nuisance_aware_and_disentangled_representation_learning}

A second related line of work seeks representations that separate task-relevant information from nuisance variation. General representation learning studies how latent variables can support abstraction, transfer, and downstream prediction~\cite{bengio2013representation}. Disentangled representation learning, including $\beta$-VAE~\cite{higgins2017betavae} and InfoGAN~\cite{chen2016infogan}, encourages latent factors to capture statistically or semantically distinct sources of variation. Invariant representation learning and domain generalization methods, including domain-adversarial learning~\cite{ganin2016domain}, invariant risk minimization~\cite{arjovsky2019invariant}, information bottleneck formulations~\cite{tishby2000information}, and related nuisance-aware objectives~\cite{moyer2018invariant,zhu2022weakly,taghanaki2021robust}, aim to suppress domain-specific or nuisance-dependent variability while retaining predictive information. Recent contrastive representation learning methods
~\cite{chen2020simple,khosla2020supervised}
construct latent geometry through positive and negative sample relationships
and have become widely used for learning discriminative representations.

These approaches address an important problem, but they define representation relevance through statistical independence, domain invariance, augmentation consistency, nuisance suppression, or contrastive similarity. In sensor-conditioned settings, however, the central question is not merely which variations should be invariant, but which scene distinctions remain justified by the sensing process after nuisance canonicalization. Excessive invariance may erase scene differences that remain distinguishable under sensing, insufficient invariance may preserve nuisance-induced measurement variation as scene structure, and contrastive objectives may organize latent geometry according to the chosen pairing strategy without explicitly encoding sensing-supported distinguishability. The proposed framework therefore differs in the representation target itself: relevance is defined through sensing-supported distinguishability after nuisance canonicalization rather than through statistical independence, invariance constraints, or contrastive similarity alone.

\subsection{Observability, Equivalence, and Quotient-Based Distinguishability}
\label{subsec:observability_equivalence_and_quotient_structure}

Classical observability theory studies whether system states can be distinguished from admissible observations~\cite{kalman1963mathematical,hermann2003nonlinear}. Related notions of observational equivalence formalize the fact that different physical states may produce indistinguishable observation histories under a given sensing process~\cite{doostmohammadian2016measurement}. The present work draws inspiration from this perspective because physical difference alone does not necessarily justify representational difference under sensing constraints. In sensing systems, distinguishability is induced jointly by the scene, the sensor, the admissible actions, and the noise model.

However, classical observability does not directly solve the representation-learning problem considered here. Its primary concern is state recoverability, sensor selection, or identifiability under an observation model. It does not specify how a learned latent representation should organize scenes when raw measurements are contaminated by nuisance factors that are visible to the sensor but irrelevant to the scene representation. If equivalence is defined directly in raw measurement space, gain drift, calibration phase offset, baseline shift, or transient interference may incorrectly become part of the representation target.

The scene-relevant observation quotient draws on the observational-equivalence viewpoint but differs from it in two important respects. First, equivalence is defined after nuisance canonicalization, so measurement-visible nuisance does not by itself justify latent separation. Second, the quotient is used as a representation target: it determines which latent distinctions should be preserved and which should collapse. Thus the proposed formulation is not simply an application of observability theory. It converts sensing-supported distinguishability into a criterion for representation correctness.

\subsection{Tensor-Structured Representation Models}
\label{subsec:tensor_structured_representation_models}

Tensor decompositions and tensor networks provide natural tools for structured sensing data because measurements often have explicit multi-way organization across time, aperture state, response mode, range bin, frequency, or spatial channel~\cite{kolda2009tensor,sidiropoulos2017tensor,cichocki2016tensor}. Tucker and related decompositions can reduce parameter count, preserve mode-specific structure, and provide an inductive bias aligned with multilinear sensing measurements. Tensor-structured autoencoding frameworks can therefore provide an effective mechanism for compressing high-dimensional sensing observations while preserving mode-specific structure.

Nevertheless, tensor structure alone does not define which latent distinctions are meaningful. A standard tensor autoencoder is typically driven by reconstruction fidelity, and reconstruction fidelity rewards preservation of all measurement-visible variation, including nuisance variation. A tensor model can therefore be structurally appropriate for the data while still geometrically incorrect as a scene representation. Conversely, adding a scene decoder to a tensor autoencoder can improve task compatibility without ensuring that the latent space contracts sensor-indistinguishable scenes or suppresses nuisance-induced false distinctions.

In the proposed framework, Tucker structure serves as an inductive bias for preserving the multilinear organization of sensing observations rather than as the primary representation objective. While the shared Tucker encoder preserves multilinear sensing structure, the central objective is to align the scene latent with a scene-relevant observation quotient. Tensor structure therefore provides an inductive bias for structured sensing observations, whereas quotient-oriented supervision defines which latent distinctions should be preserved and which should collapse.

\subsection{Positioning of the Present Work}
\label{subsec:related_work_positioning}

Existing sensing representations, invariant and disentangled objectives, observability-based distinguishability analysis, and tensor-structured models provide important ingredients for sensor-conditioned learning. However, none explicitly defines a representation target through sensing-supported distinguishability after nuisance canonicalization. The scene-relevant observation quotient addresses this gap by defining representation correctness at the level of distinguishability structure, while OQ-TSAE provides a practical realization of this principle through quotient-oriented supervision and scene--nuisance factorization.

\section{Sensor-Conditioned Scene-State Formulation}
\label{sec:sensor_conditioned_scene_state_formulation}

This section formalizes the scene-relevant observation quotient underlying the proposed representation target. The central premise is that scene equivalence should be determined not by raw measurement variability alone, nor by physical scene identity alone, but by sensing-supported distinguishability after nuisance canonicalization. In particular, physically different scene configurations may still be considered equivalent whenever their differences are not reliably distinguishable through the sensing process after nuisance removal. Conversely, scene differences that remain resolvable under sensing should be preserved. Figure~\ref{fig:quotient_intuition} provides an intuitive overview of this principle before introducing the formal sensing model, nuisance canonicalization operator, and quotient construction.

\begin{figure*}
    \centering
    \includegraphics[width=\textwidth]{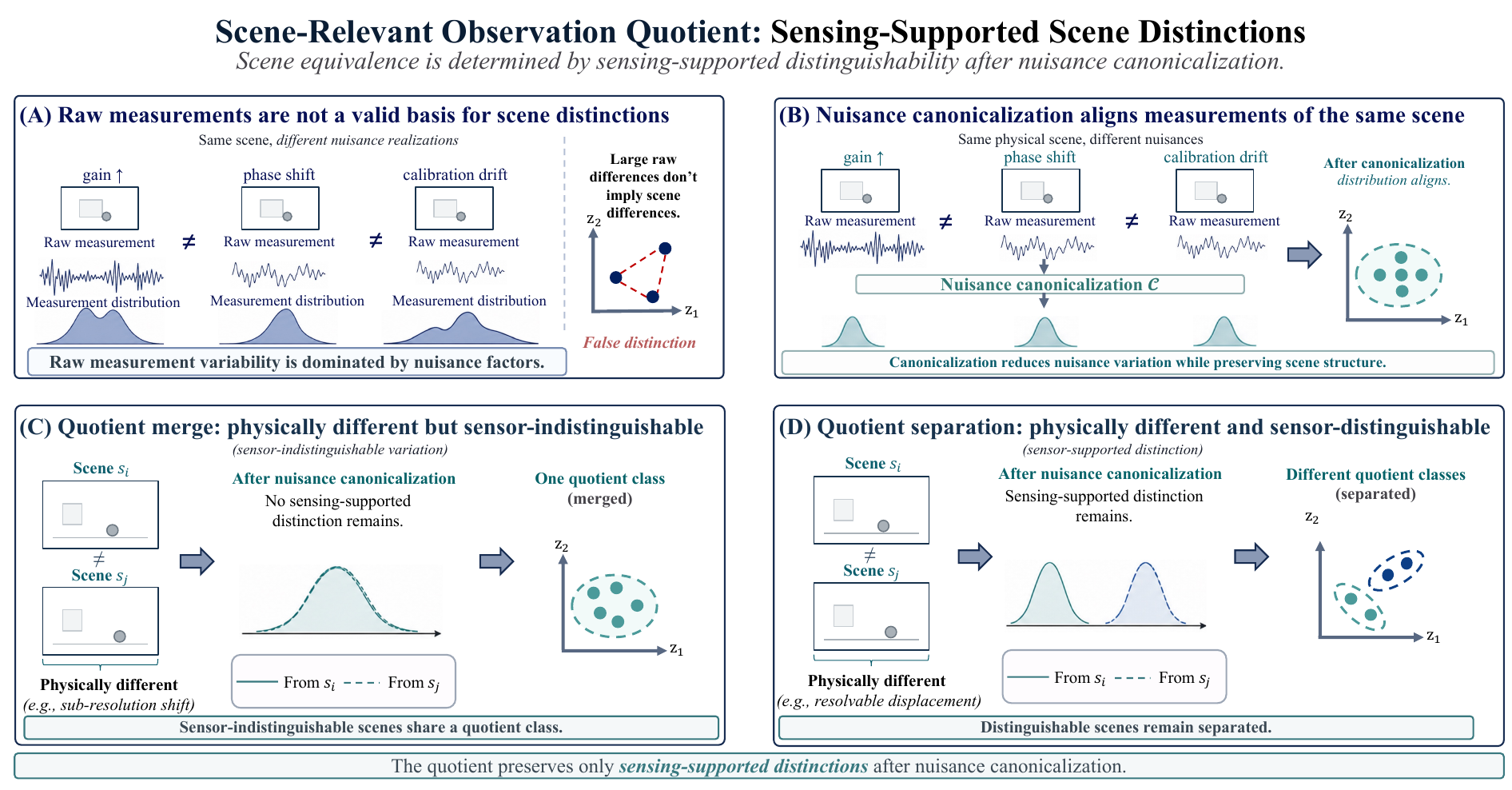}
    \caption{
    Conceptual illustration of the scene-relevant observation quotient. After nuisance canonicalization, scene configurations are grouped according to sensing-supported distinguishability: sensor-indistinguishable configurations belong to the same quotient class, whereas sensor-resolvable differences remain distinct.
    }
    \label{fig:quotient_intuition}
\end{figure*}

Figure~\ref{fig:quotient_intuition} serves as an intuitive bridge between the motivating failure modes of representation correctness and the formal construction introduced below. The following subsections formalize this intuition through a scene--nuisance decomposition, an action-indexed sensing model, nuisance canonicalization, and the resulting scene-relevant observation quotient.

\subsection{Physical Configuration and the Scene--Nuisance Decomposition}
\label{subsec:physical_configuration_and_the_scene_nuisance_deco}

We begin by defining the space of physical configurations encountered by the sensing platform. At each time step $t$, the physical scene state is described by a collection of target parameters:

\begin{equation}
\theta_t^{\phys}
=
\left\{
p_{\ell,t},
v_{\ell,t},
\alpha_{\ell,t},
m_{\ell,t}
\right\}_{\ell=1}^{L},
\label{eq:physical_state}
\end{equation}
where $p_{\ell,t} \in \R^{d_p}$ denotes the position of the $\ell$-th scatterer or target, $v_{\ell,t} \in \R^{d_v}$ denotes its motion state, $\alpha_{\ell,t} \in \C$ is its complex-valued reflectivity coefficient capturing both amplitude and phase of the scattered field, and $m_{\ell,t}$ is an optional discrete material or object class label. The number of targets $L$ may vary across scenes. The space of all such configurations is denoted by $\thetaspace$.
In the controlled benchmark of Section~\ref{sec:controlled_sensor_conditioned_benchmark}, we instantiate this general parameterization through a controlled sensor-conditioned benchmark with benchmark-specific scene and motion parameterizations to facilitate mechanism-level analysis under known sensing conditions.

A central modeling assumption of this work is the decomposition of the physical configuration into two components that serve distinct roles in representation learning:

\begin{equation}
\theta_t^{\phys}
=
(s_t,\eta_t),
\label{eq:decomposition}
\end{equation}
where $s_t$ denotes the scene variables, i.e., the aspects of the physical scene that the representation is intended to preserve and on which downstream tasks depend, and $\eta_t$ denotes nuisance variables, namely measurement-visible factors that affect raw sensor observations without reflecting changes in the scene itself.

The scene variables $s_t$ typically include geometric and electromagnetic target properties that downstream tasks depend upon, such as target positions, motion characteristics, reflectivity values, and material attributes when task-relevant. The precise scene-state parameterization may vary across sensing settings depending on sensing capability and observable resolution. The nuisance variables $\eta_t$ capture hardware-related and environmental effects that modulate measurements independently of scene changes. In structured electromagnetic sensing, representative nuisance factors include receiver gain drift $g_t^{\rx}$, calibration phase offset $\phi_t^{\calib}$, baseline offset $b_t^{\offset}$, and transient interference $c_t^{\transient}$.

This decomposition formalizes the distinction between scene structure and measurement-visible nuisance. For example, two sensing episodes with identical scene geometry and reflectivity but different receiver gain may produce substantially different raw observations despite corresponding to the same underlying scene. A representation driven solely by raw measurement distinguishability may therefore encode hardware state rather than scene state. By distinguishing scene variables $s_t$ from nuisance variables $\eta_t$, we obtain the variables on which the scene-relevant observation quotient is defined.

\subsection{Controllable Sensing Actions and Measurement Structure}
\label{subsec:controllable_sensing_actions_and_measurement_struc}

A defining characteristic of structured sensing platforms is that the sensing agent exercises control over specific degrees of freedom in the measurement process~\cite{li2008mimo,van2004optimum}. These controllable degrees of freedom constitute the sensing actions. We model each sensing action as a pair:

\begin{equation}
a^{\act}
=
(m,r)
\in
\cU_{\act}
=
\mathcal{M}
\times
\mathcal{R},
\label{eq:sensing_action}
\end{equation}
where $m \in \mathcal{M}$ indexes the aperture state, which may correspond to a transmit--receive element pair in a MIMO array, a beam index in a phased-array system, or a frequency-indexed spatial mode in a frequency-diverse aperture. The second coordinate $r \in \mathcal{R}$ indexes the response mode, such as a receive channel, polarization state, or orthogonal coding dimension. The set $\cU_{\act}$ of admissible sensing actions constitutes the action family of the platform.

When the sensing agent executes action $a^{\act}$, the sensor returns a vector of measurements over range or beat-frequency bins:

\begin{equation}
\mathbf{y}(a^{\act};\theta_t^{\phys})
=
\big[
y(m,r,q;\theta_t^{\phys})
\big]_{q\in\mathcal{Q}}
\in
\C^{Q},
\label{eq:measurement_vector}
\end{equation}
where $q \in \mathcal{Q}$ indexes the range bin determined by propagation delay or beat frequency. Importantly, the action coordinates $(m,r)$ are selected by the sensing agent, whereas the measurement coordinate $q$ is produced by the physical measurement process. Thus, the sensing agent selects an aperture state and response mode, while the sensor returns a range-resolved measurement vector.

Over a temporal window of length $K+1$, the resulting measurements are organized into a four-dimensional observation tensor:

\begin{equation}
\cY_{t-K:t}
=
\Big[
y(m,r,q,\tau)
\Big]_
{\substack{
\tau=t-K,\ldots,t\\
m\in\mathcal{M},
r\in\mathcal{R},
q\in\mathcal{Q}
}}
\in
\C^{(K+1)\times M\times R\times Q}.
\label{eq:obs_tensor}
\end{equation}

Each tensor mode carries explicit physical meaning. The temporal mode of dimension $T=K+1$ captures scene evolution and encodes Doppler information through phase progression across frames. The aperture-index mode of dimension $M$ provides spatial diversity and governs angular resolution. The response mode of dimension $R$ captures channel or polarization diversity, while the range-bin mode of dimension $Q$ provides depth resolution and target separation along the line of sight. This organization is not an arbitrary data format but reflects the physical sensing degrees of freedom of the platform and motivates the Tucker-structured representation introduced in Section~\ref{sec:scene_nuisance_factorized_oq_tsae}.

\subsection{Forward Observation Model}
\label{subsec:forward_observation_model}

For a given physical configuration $\theta_t^{\phys}$ and sensing action $a^{\act}$, the noise-free observation follows a standard multi-scatterer radar/array forward model~\cite{li2008mimo,van2004optimum,skolnik2008radar}:

\begin{equation}
\begin{aligned}
\cH(m,r,q;\theta_t^{\phys})
&=
\sum_{\ell=1}^{L}
\alpha_{\ell,t}
\,G_{m,r}(p_{\ell,t},q)
\\
&\quad \times
\exp\!\left(
-j
\frac{4\pi f_m}{c}
R_{m,r}(p_{\ell,t},q)
\right),
\end{aligned}
\label{eq:forward_model}
\end{equation}
where $G_{m,r}(p,q)$ denotes the platform-dependent spatial response function incorporating beam patterns, aperture weights, and propagation effects, $R_{m,r}(p,q)$ is the corresponding round-trip propagation path length, $f_m$ is the operating frequency associated with aperture state $m$, and $c$ is the speed of light. The complete noise-free response vector for action $a^{\act}$ is

\begin{equation}
\mathbf{h}(a^{\act};\theta_t^{\phys})
=
\big[
\cH(m,r,q;\theta_t^{\phys})
\big]_{q\in\mathcal{Q}}.
\label{eq:noisefree_response}
\end{equation}

Observed measurements are corrupted by additive noise. We adopt a general measurement model in which the noise vector associated with action $a^{\act}$ follows a circularly symmetric complex Gaussian distribution~\cite{steven1993fundamentals} with covariance matrix $\Sigma_a$:

\begin{equation}
\mathbf{y}(a^{\act})
=
\mathbf{h}(a^{\act};\theta_t^{\phys})
+
\mathbf{n}_a,
\qquad
\mathbf{n}_a
\sim
\mathcal{CN}(0,\Sigma_a),
\label{eq:noise_model}
\end{equation}
where $\Sigma_a$ may capture both thermal noise and action-dependent uncertainty, such as frequency-dependent phase noise. In the primary formulation, $\Sigma_a$ is treated as available from calibration; sensitivity to imperfect calibration is examined experimentally.

\subsection{Nuisance Canonicalization}
\label{subsec:nuisance_canonicalization}

Raw measurement distributions generally depend jointly on scene variables and nuisance variables. To isolate the scene-relevant component of the observation process, we introduce a nuisance canonicalization operator:

\begin{equation}
\calop:
\cY
\mapsto
\widetilde{\cY},
\label{eq:canonicalization}
\end{equation}
which maps raw observations to canonicalized observations through removal or normalization of known nuisance degrees of freedom. In structured electromagnetic sensing, representative canonicalization operations include global gain normalization to reduce receiver gain drift, common phase alignment to mitigate calibration phase offsets, baseline subtraction using calibration measurements, and background subtraction to suppress static clutter returns.

The canonicalized observation $\widetilde{\cY}$ is intended to preserve scene-induced measurement structure while suppressing nuisance-driven variation. As a reference condition, we define ideal canonicalization as satisfying nuisance invariance at the distribution level: for any scene configuration $s$ and any two nuisance realizations $\eta,\eta'$,

\begin{equation}
p(
\widetilde{\cY}
\mid
s,\eta,a
)
=
p(
\widetilde{\cY}
\mid
s,\eta',a
),
\qquad
\forall
a
\in
\cU_{\act}.
\label{eq:canonicalization_condition}
\end{equation}

This condition characterizes an ideal nuisance canonicalization operator. In practice, canonicalization is generally imperfect and may leave residual nuisance-dependent variation. The theoretical analysis in Section~\ref{subsec:canonicalization_stability} therefore studies the stability of quotient assignment under bounded canonicalization error rather than assuming exact invariance.

\subsection{Scene-Relevant Observation Quotient}
\label{subsec:scene_relevant_observation_quotient}

With the nuisance canonicalization operator in place, we define the scene-relevant observation equivalence relation that underlies the representation objective of this work.

\begin{definition}[Scene-Relevant Observation Equivalence]
Two scene configurations $s_i, s_j \in \mathcal{S}$ are observationally equivalent under the sensing action family $\cU_{\act}$ and canonicalization operator $\calop$, denoted by $s_i \simU s_j$, if and only if they induce identical canonicalized measurement distributions under all admissible sensing actions:

\begin{equation}
p(
\widetilde{\cY}
\mid
s_i,a
)
=
p(
\widetilde{\cY}
\mid
s_j,a
),
\qquad
\forall
a
\in
\cU_{\act}.
\label{eq:scene_equiv}
\end{equation}

\end{definition}

The relation $\simU$ defines an equivalence relation on the scene space $\mathcal{S}$. Reflexivity and symmetry follow directly from the definition, while transitivity follows from the transitivity of equality of probability distributions. The equivalence class associated with a scene configuration $s$ is

\begin{equation}
[s]
=
\{
s'
\in
\mathcal{S}
:
s'
\simU
s
\},
\label{eq:equivalence_class}
\end{equation}
representing the set of scene configurations that remain observationally indistinguishable from $s$ under the sensing action family after nuisance canonicalization.

\begin{definition}[Scene-Relevant Observation Quotient]
The scene-relevant observation quotient space is the set of equivalence classes:

\begin{equation}
\Qscene
=
\mathcal{S}
/\simU.
\label{eq:quotient_space}
\end{equation}

\end{definition}

An element of $\Qscene$ represents an equivalence class of scene configurations rather than a single physical realization. The quotient space formalizes the observable resolution of the sensing platform by identifying which distinctions among physical scene configurations are supported by the sensing process after nuisance canonicalization and which are observationally inaccessible.

\subsection{Scene-Relevant Distinguishability Metric}
\label{subsec:scene_relevant_distinguishability_metric}

For the observation quotient to guide representation learning, we require a computable measure of scene distinguishability. Let $\widetilde{\mathbf h}(a;s)$ denote the canonicalized noise-free response vector for scene $s$ under sensing action $a$, obtained by evaluating the forward model and applying the nuisance canonicalization operator. We define the scene-relevant distinguishability metric as

\begin{equation}
\begin{aligned}
d_{\Sigma,\mathrm{scene}}^2(s_i,s_j)
&=
\sum_{a\in\mathcal{U}_{\mathrm{act}}}
\Delta\widetilde{\mathbf h}_{a,ij}^{H}
\Sigma_a^{-1}
\Delta\widetilde{\mathbf h}_{a,ij},
\\
\Delta\widetilde{\mathbf h}_{a,ij}
&=
\widetilde{\mathbf h}(a;s_i)
-
\widetilde{\mathbf h}(a;s_j).
\end{aligned}
\label{eq:dsigma_scene}
\end{equation}

This distinguishability functional incorporates the noise characteristics of the sensing platform through inverse covariance weighting, so that response differences under noisier sensing conditions contribute less to distinguishability. Strictly speaking, $d_{\Sigma,\mathrm{scene}}$ behaves as a pseudometric on the original scene space $\mathcal S$, since distinct scene configurations may induce identical canonicalized sensing responses and therefore satisfy $d_{\Sigma,\mathrm{scene}}(s_i,s_j)=0$. A proper metric is obtained only after quotienting by scene equivalence. This behavior is intentional rather than pathological: zero distinguishability identifies scene configurations that should collapse under the scene-relevant observation quotient. Under Gaussian measurement noise, Theorem~\ref{thm:consistency} establishes that zero scene-relevant distinguishability is equivalent to scene-relevant observational equivalence.

For representation learning, the continuous distinguishability metric is converted into reliable pairwise supervision sets through thresholding. Since thresholded distance relations on finite data are generally not transitive and therefore do not form exact equivalence classes, we refer to these sets as reliable supervision pairs rather than quotient classes. We define three categories of scene pairs:

\begin{align}
\mathcal{P}_{\eq}^{\scene}
&=
\{
(i,j)
:
d_{\Sigma,\mathrm{scene}}(s_i,s_j)
\le
\epsilon_{\mathrm{eq}}
\},
\label{eq:peq_scene}
\\
\mathcal{P}_{\dist}^{\scene}
&=
\{
(i,j)
:
d_{\Sigma,\mathrm{scene}}(s_i,s_j)
\ge
\epsilon_{\mathrm{dist}}
\},
\label{eq:pdist_scene}
\\
\mathcal{P}_{\amb}^{\scene}
&=
\{
(i,j)
:
\epsilon_{\mathrm{eq}}
<
d_{\Sigma,\mathrm{scene}}(s_i,s_j)
<
\epsilon_{\mathrm{dist}}
\},
\label{eq:pamb_scene}
\end{align}
where $\mathcal{P}_{\eq}^{\scene}$ contains observationally similar scene pairs intended for contraction, $\mathcal{P}_{\dist}^{\scene}$ contains reliably distinguishable scene pairs intended for separation, and $\mathcal{P}_{\amb}^{\scene}$ contains ambiguous pairs excluded from hard supervision. The values of $\epsilon_{\eq}$ and $\epsilon_{\dist}$ used in the primary experiments are obtained from the calibration protocol used for reliable pair construction and are reported in Appendix~\ref{app:hyperparameters}.

\subsection{Observable Scene Targets}
\label{subsec:observable_scene_targets}

The final component of the problem formulation specifies the observable scene quantities that the learned representation is intended to recover. These targets are defined at a sensor-compatible resolution to ensure consistency with the scene-relevant observation quotient.

The primary observable scene target is defined as
\begin{equation}
s_t^{\obs}
=
\big(
\widetilde{\boldsymbol{\rho}}_t^{\res},
\;
\boldsymbol{\upsilon}_t^{\res}
\big),
\label{eq:obs_target}
\end{equation}
where
\begin{equation}
\widetilde{\boldsymbol{\rho}}_t^{\res}
=
\calN_{\calib}
(
\boldsymbol{\rho}_t^{\res}
)
\in
\C^{N_x^{\res}
\times
N_y^{\res}
\times
N_z^{\res}}
\end{equation}
denotes the calibrated relative complex reflectivity field. The normalization operator $\calN_{\calib}$ removes global amplitude scaling and global phase rotation while preserving relative cell-wise amplitude and phase structure. This choice reflects the fact that absolute reflectivity amplitude and absolute phase are generally confounded with nuisance factors such as receiver gain drift and calibration phase offset in the absence of an external reference. By targeting relative rather than absolute reflectivity, the observable target remains compatible with nuisance invariance.

The second component,
\begin{equation}
\boldsymbol{\upsilon}_t^{\res}
\in
\R^{
2
\times
N_x^{\res}
\times
N_y^{\res}
\times
N_z^{\res}
},
\end{equation}
represents a two-component planar velocity field defined on the observable spatial grid and captures observable Doppler-related motion dynamics. In the present benchmark, motion supervision is restricted to planar velocity components $(v_x,v_y)$ to match the controlled sensing setting and sensor-compatible observable motion structure. Motion information serves as an auxiliary observable scene descriptor rather than a separate scene-pair perturbation family. This choice reflects the benchmark instantiation rather than a restriction of the proposed formulation.

Standard downstream outputs may be obtained through deterministic post-processing of the observable fields. For example, an occupancy field may be derived through reflectivity thresholding,
\begin{equation}
\cO_t^{\res}
=
\mathbf{1}
\!\left[
|
\widetilde{\boldsymbol{\rho}}_t^{\res}
|
>
\tau_\rho
\right],
\end{equation}
while target localization may be performed through peak detection on the reflectivity magnitude field.

Observable targets must be quotient-compatible: they should not assign different supervision targets to scene configurations that are indistinguishable under the canonicalized sensing family. Since observable scene targets are defined at a finite sensor-compatible resolution, quotient compatibility is imposed after observable rendering rather than at the level of exact physical scene identity.

Formally, let
\[
\Pi_{\res}:
\mathcal S
\rightarrow
\mathcal S^{\obs}
\]
denote a sensor-compatible rendering operator that maps physical scene configurations to observable scene targets at finite spatial resolution. We require the induced target map
\[
s
\mapsto
s^{\obs}
=
\Pi_{\res}(s)
\]
to be constant on equivalence classes of $\Qscene$. Under this requirement, if
\[
s_i \simU s_j,
\]
then
\[
\Pi_{\res}(s_i)
=
\Pi_{\res}(s_j).
\]

Thus the observable target defines a well-defined function on the quotient space $\Qscene$. This condition should be interpreted as a target-design requirement ensuring compatibility between supervision and sensing-supported distinguishability rather than as a recoverability guarantee for arbitrary physical scene labels.

\section{Theoretical Analysis}
\label{sec:theoretical_analysis}

This section develops the theoretical foundations underlying the proposed quotient-consistent representation framework. We first establish the consistency of the scene-relevant distinguishability metric with the scene-relevant observation quotient, showing that quotient equivalence can be characterized through a computable metric on canonicalized observations. We then analyze the stability of quotient assignment under imperfect nuisance canonicalization, showing that bounded canonicalization error induces only bounded perturbations in scene-relevant distinguishability and reliable supervision assignment. Next, we study the relationship between the proposed learning objectives and the empirical diagnostic criteria used to assess quotient consistency. Finally, we present an information-geometric interpretation of the scene-relevant distinguishability metric through its connection to statistical distances on canonicalized measurement distributions.

\subsection{Quotient Consistency}
\label{subsec:quotient_consistency}

\begin{theorem}[Characterization of Quotient Equivalence]
\label{thm:consistency}
Assume that the measurement noise follows the Gaussian model of Eq.~\eqref{eq:noise_model} with positive definite covariance matrices $\Sigma_a \succ 0$ for all $a \in \cU_{\act}$, and that the canonicalization operator $\calop$ satisfies the ideal canonicalization condition of Eq.~\eqref{eq:canonicalization_condition}. Then the scene-relevant distinguishability metric characterizes the scene-relevant observation quotient in the sense that, for any two scene configurations $s_i,s_j\in\mathcal S$,
\begin{equation}
d_{\Sigma,\mathrm{scene}}(s_i,s_j)=0
\iff
s_i \simU s_j .
\label{eq:complete_invariant}
\end{equation}
\end{theorem}

\begin{proof}
We prove the two directions separately.

First, suppose $d_{\Sigma,\mathrm{scene}}(s_i,s_j)=0$. By Eq.~\eqref{eq:dsigma_scene}, the distinguishability metric is a sum of non-negative action-wise quadratic terms. Since the sum equals zero, each term must be zero individually. Since $\Sigma_a \succ 0$, its inverse $\Sigma_a^{-1}$ is also positive definite, and therefore
\[
\Delta\widetilde{\mathbf h}_{a,ij}^{H}
\Sigma_a^{-1}
\Delta\widetilde{\mathbf h}_{a,ij}
=0
\]
if and only if $\Delta\widetilde{\mathbf h}_{a,ij}=0$. Hence
$\widetilde{\mathbf h}(a;s_i)=\widetilde{\mathbf h}(a;s_j)$ for all $a\in\cU_{\act}$. Under the Gaussian noise model with common action-specific covariance $\Sigma_a$, equality of canonicalized mean responses implies equality of the canonicalized measurement distributions for every admissible action. Therefore,
\[
p(\widetilde{\cY}\mid s_i,a)
=
p(\widetilde{\cY}\mid s_j,a),
\qquad
\forall a\in\cU_{\act},
\]
which is precisely $s_i\simU s_j$ by Eq.~\eqref{eq:scene_equiv}.

Conversely, suppose $s_i\simU s_j$. Then the canonicalized measurement distributions induced by $s_i$ and $s_j$ are identical for all admissible actions. Under the Gaussian model with the same covariance $\Sigma_a$ for a fixed action, equality of distributions implies equality of the mean vectors:
\[
\widetilde{\mathbf h}(a;s_i)
=
\widetilde{\mathbf h}(a;s_j),
\qquad
\forall a\in\cU_{\act}.
\]
Thus $\Delta\widetilde{\mathbf h}_{a,ij}=0$ for every action, and Eq.~\eqref{eq:dsigma_scene} gives
$d_{\Sigma,\mathrm{scene}}(s_i,s_j)=0$.
\end{proof}

Theorem~\ref{thm:consistency} establishes that the scene-relevant distinguishability metric is a complete characterization of quotient equivalence under the assumed Gaussian sensing model and ideal canonicalization condition. This result provides theoretical justification for using $d_{\Sigma,\mathrm{scene}}$ as the metric basis for constructing the reliable pairwise supervision sets in Eqs.~\eqref{eq:peq_scene}--\eqref{eq:pamb_scene}.

\subsection{Canonicalization Stability}
\label{subsec:canonicalization_stability}

In practice, the canonicalization operator may be estimated from calibration data or approximated from observations rather than given exactly. Let $\widehat{\calop}$ denote an approximate canonicalization operator, and let $\widehat d_{\Sigma,\mathrm{scene}}$ denote the distinguishability metric computed using $\widehat{\calop}$ in place of the ideal operator $\calop$. Since nuisance canonicalization is generally estimated rather than ideal, an important question is whether quotient assignments remain stable under canonicalization error. The following result shows that bounded canonicalization error induces only bounded perturbations in scene-relevant distinguishability. Throughout this subsection, we use the weighted norm
\[
\|x\|_{\Sigma_a^{-1}}
=
\big(
x^H\Sigma_a^{-1}x
\big)^{1/2}.
\]

\begin{theorem}[Canonicalization Stability]
\label{thm:stability}
Assume that there exists $\delta_{\calop} \ge 0$ such that for every scene configuration $s \in \mathcal{S}$ and every action $a \in \cU_{\act}$, the canonicalized response computed under $\widehat{\calop}$ differs from the ideal canonicalized response by at most $\delta_{\calop}$ in the $\Sigma_a^{-1}$-norm:
\begin{equation}
\big\|
\widehat{\widetilde{\mathbf h}}(a;s)
-
\widetilde{\mathbf h}(a;s)
\big\|_{\Sigma_a^{-1}}
\le
\delta_{\calop}.
\label{eq:canon_error_bound}
\end{equation}
Then, for any two scene configurations $s_i,s_j \in \mathcal{S}$,
\begin{equation}
\big|
\widehat d_{\Sigma,\mathrm{scene}}(s_i,s_j)
-
d_{\Sigma,\mathrm{scene}}(s_i,s_j)
\big|
\le
2\delta_{\calop}\sqrt{|\cU_{\act}|}.
\label{eq:stability_bound}
\end{equation}
\end{theorem}

\begin{proof}
Let
\[
\widehat{\Delta}_{a,ij}
=
\widehat{\widetilde{\mathbf h}}(a;s_i)
-
\widehat{\widetilde{\mathbf h}}(a;s_j),
\qquad
\Delta_{a,ij}
=
\widetilde{\mathbf h}(a;s_i)
-
\widetilde{\mathbf h}(a;s_j).
\]
For each action $a$, the reverse triangle inequality in the $\Sigma_a^{-1}$-norm gives
\begin{equation}
\big|
\|\widehat{\Delta}_{a,ij}\|_{\Sigma_a^{-1}}
-
\|\Delta_{a,ij}\|_{\Sigma_a^{-1}}
\big|
\le
\|\widehat{\Delta}_{a,ij}-\Delta_{a,ij}\|_{\Sigma_a^{-1}}.
\end{equation}
The difference between the estimated and ideal response differences is bounded by the two per-scene canonicalization errors:
\begin{equation}
\begin{aligned}
\|\widehat{\Delta}_{a,ij}-\Delta_{a,ij}\|_{\Sigma_a^{-1}}
&\le
\big\|
\widehat{\widetilde{\mathbf h}}(a;s_i)
-
\widetilde{\mathbf h}(a;s_i)
\big\|_{\Sigma_a^{-1}}
\\
&\quad+
\big\|
\widehat{\widetilde{\mathbf h}}(a;s_j)
-
\widetilde{\mathbf h}(a;s_j)
\big\|_{\Sigma_a^{-1}}
\\
&\le
2\delta_{\calop}.
\end{aligned}
\end{equation}
Now view $d_{\Sigma,\mathrm{scene}}(s_i,s_j)$ as the Euclidean norm of the vector whose entries are
$\|\Delta_{a,ij}\|_{\Sigma_a^{-1}}$ over $a\in\cU_{\act}$, and similarly for $\widehat d_{\Sigma,\mathrm{scene}}(s_i,s_j)$. Applying the reverse triangle inequality to these action-wise norm vectors yields
\begin{equation}
\begin{aligned}
&
\big|
\widehat d_{\Sigma,\mathrm{scene}}(s_i,s_j)
-
d_{\Sigma,\mathrm{scene}}(s_i,s_j)
\big|
\\
&\le
\left(
\sum_{a\in\cU_{\act}}
(2\delta_{\calop})^2
\right)^{1/2}
\\
&=
2\delta_{\calop}
\sqrt{|\cU_{\act}|}.
\end{aligned}
\end{equation}
\end{proof}

\begin{corollary}[Reliable Supervision Stability]
\label{cor:assignment_stability}
Let $\epsilon_{\mathrm{eq}}<\epsilon_{\mathrm{dist}}$ be the thresholds defining the reliable pairwise supervision sets in Eqs.~\eqref{eq:peq_scene}--\eqref{eq:pamb_scene}, and define
\[
\rho_{\calop}
=
2\delta_{\calop}\sqrt{|\cU_{\act}|}.
\]
If $\rho_{\calop}<(\epsilon_{\mathrm{dist}}-\epsilon_{\mathrm{eq}})/2$, then no pair classified as equivalent under the ideal metric can be classified as distinguishable under the estimated metric, and no pair classified as distinguishable under the ideal metric can be classified as equivalent under the estimated metric. Pairs within or near the ambiguous region may still move into or out of the ambiguous band, but no reliable equivalent or distinguishable supervision pair is reversed.
\end{corollary}

\begin{proof}
By Theorem~\ref{thm:stability}, every pairwise distinguishability value changes by at most $\rho_{\calop}$. If $(i,j)\in\mathcal{P}_{\eq}^{\scene}$ under the ideal metric, then
\[
d_{\Sigma,\mathrm{scene}}(s_i,s_j)\le \epsilon_{\mathrm{eq}},
\]
and therefore
\[
\widehat d_{\Sigma,\mathrm{scene}}(s_i,s_j)
\le
\epsilon_{\mathrm{eq}}
+
\rho_{\calop}
<
\epsilon_{\mathrm{dist}}.
\]
Thus the pair cannot be classified as distinguishable under the estimated metric. Similarly, if $(i,j)\in\mathcal{P}_{\dist}^{\scene}$ under the ideal metric, then
\[
d_{\Sigma,\mathrm{scene}}(s_i,s_j)\ge \epsilon_{\mathrm{dist}},
\]
and hence
\[
\widehat d_{\Sigma,\mathrm{scene}}(s_i,s_j)
\ge
\epsilon_{\mathrm{dist}}
-
\rho_{\calop}
>
\epsilon_{\mathrm{eq}}.
\]
Thus the pair cannot be classified as equivalent under the estimated metric.
\end{proof}

This result shows that bounded canonicalization error cannot reverse reliable supervision decisions when the error is small relative to the gap between the equivalence and distinguishability thresholds. Pairs near the threshold boundaries may still move into or out of the ambiguous band, but such pairs are excluded from hard supervision. The stability result therefore supports the use of estimated canonicalization in settings where the induced perturbation of $d_{\Sigma,\mathrm{scene}}$ remains below the reliable supervision margin.

Consequently, quotient-oriented supervision remains well defined under imperfect nuisance canonicalization. As long as the canonicalization error remains below the supervision margin, reliable equivalent and distinguishable pairs remain correctly assigned.

\subsection{Diagnostic Consistency of Learning Objectives}
\label{subsec:diagnostic_consistency_of_learning_objectives}
The proposed diagnostics are intended to evaluate quotient consistency, whereas the learning objectives are used to optimize it. An important question is therefore whether improvements in the learning objectives translate into improvements in the diagnostic criteria.

The learning objectives introduced in Section~\ref{sec:scene_nuisance_factorized_oq_tsae} are designed to encourage a latent geometry that is consistent with the scene-relevant quotient structure. The following result shows that reducing the contraction and separation objectives also controls the corresponding diagnostic error rates used to assess representation correctness.

\begin{theorem}[Diagnostic Consistency of Quotient Geometry]
\label{thm:soundness}
Let
\[
\bar z^{\scene}:\mathcal S\to\mathbb S^{d-1}
\]
denote the normalized scene embedding. Suppose that the empirical contraction and separation losses
in Eqs.~\eqref{eq:L_eq}--\eqref{eq:L_sep}
satisfy
\[
\mathcal L_{\eq}\le \alpha_{\eq},
\qquad
\mathcal L_{\dist}\le \alpha_{\dist}.
\]
For any false-distinction diagnostic threshold $\delta_{\eq}>0$,
\begin{equation}
\operatorname{FDR}(\delta_{\eq})
\le
\frac{\alpha_{\eq}}{\delta_{\eq}^2}.
\label{eq:fdr_bound_revised}
\end{equation}
Moreover, for any false-merge diagnostic threshold $\delta_{\dist}<m_{\dist}$,
\begin{equation}
\operatorname{FMR}(\delta_{\dist})
\le
\frac{\alpha_{\dist}}{(m_{\dist}-\delta_{\dist})^2}.
\label{eq:fmr_bound_revised}
\end{equation}
Consequently, if $\alpha_{\eq}\to0$ and $\alpha_{\dist}\to0$, then the corresponding false distinction and false merge rates vanish for fixed diagnostic thresholds satisfying $\delta_{\dist}<m_{\dist}$.
\end{theorem}

\begin{proof}
For equivalent pairs, Markov's inequality applied to the non-negative random variable
\[
\left\|
\bar z_i^{\scene}
-
\bar z_j^{\scene}
\right\|_2^2,
\qquad
(i,j)\in\mathcal P_{\eq}^{\scene},
\]
gives
\[
\begin{aligned}
\operatorname{FDR}(\delta_{\eq})
&=
\Pr\!\left[
\left\|
\bar z_i^{\scene}
-
\bar z_j^{\scene}
\right\|_2
>
\delta_{\eq}
\;\middle|\;
(i,j)\in\mathcal P_{\eq}^{\scene}
\right]
\\
&\le
\frac{
\mathbb E_{\mathcal P_{\eq}^{\scene}}
\left[
\left\|
\bar z_i^{\scene}
-
\bar z_j^{\scene}
\right\|_2^2
\right]
}{
\delta_{\eq}^2
}
=
\frac{\mathcal L_{\eq}}{\delta_{\eq}^2}
\le
\frac{\alpha_{\eq}}{\delta_{\eq}^2}.
\end{aligned}
\]

For distinguishable pairs, define
\[
r_{ij}
=
\left\|
\bar z_i^{\scene}
-
\bar z_j^{\scene}
\right\|_2 .
\]
If $r_{ij}\le\delta_{\dist}$ and $\delta_{\dist}<m_{\dist}$, then
\[
[m_{\dist}-r_{ij}]_+^2
\ge
(m_{\dist}-\delta_{\dist})^2 .
\]
Applying Markov's inequality to the non-negative separation-violation variable
\[
[m_{\dist}-r_{ij}]_+^2,
\qquad
(i,j)\in\mathcal P_{\dist}^{\scene},
\]
yields
\[
\begin{aligned}
\operatorname{FMR}(\delta_{\dist})
&=
\Pr\!\left[
r_{ij}
\le
\delta_{\dist}
\;\middle|\;
(i,j)\in\mathcal P_{\dist}^{\scene}
\right]
\\
&\le
\frac{
\mathbb E_{\mathcal P_{\dist}^{\scene}}
\left[
[m_{\dist}-r_{ij}]_+^2
\right]
}{
(m_{\dist}-\delta_{\dist})^2
}
=
\frac{\mathcal L_{\dist}}{(m_{\dist}-\delta_{\dist})^2}
\\
&\le
\frac{\alpha_{\dist}}{(m_{\dist}-\delta_{\dist})^2}.
\end{aligned}
\]
Combining the above inequalities with the assumed bounds
$\mathcal L_{\eq}\le\alpha_{\eq}$
and
$\mathcal L_{\dist}\le\alpha_{\dist}$
establishes the stated result.
\end{proof}

Theorem~\ref{thm:soundness} establishes that reducing the contraction and separation objectives also reduces the empirical diagnostic error rates used to evaluate representation correctness. Although the theorem does not guarantee perfect quotient recovery in finite data, it provides theoretical support for the use of the contraction and separation losses as surrogates for the representation-level diagnostics reported in Section~\ref{sec:experimental_evaluation}.

\subsection{Information-Geometric Interpretation}
\label{subsec:information_geometric_interpretation}

Although not required for the quotient construction itself, the scene-relevant distinguishability metric also admits an information-geometric interpretation~\cite{amari2016information}. We present this connection to clarify how the proposed metric relates to statistical distances between canonicalized measurement distributions and local statistical geometry induced by canonicalized sensing distributions.

\begin{theorem}[Connection to Symmetric KL Divergence]
\label{thm:fisher_rao}
Under the Gaussian noise model of Eq.~\eqref{eq:noise_model}, assume that the action-wise noise covariance is action-independent, i.e., $\Sigma_a=\Sigma$ for all $a\in\cU_{\act}$. Let
\[
p_i
=
\prod_{a\in\cU_{\act}}
\mathcal{CN}
\big(
\widetilde{\mathbf h}(a;s_i),
\Sigma
\big)
\]
denote the product distribution of canonicalized measurements induced by scene $s_i$ over all admissible actions. Then the squared scene-relevant distinguishability metric is proportional to the symmetric Kullback--Leibler divergence:
\begin{equation}
d_{\Sigma,\mathrm{scene}}^2(s_i,s_j)
\propto
D_{\mathrm{KL}}(p_i\|p_j)
+
D_{\mathrm{KL}}(p_j\|p_i),
\label{eq:fisher_rao_eq}
\end{equation}
where the proportionality constant depends only on the Gaussian convention adopted for the complex-valued measurement model.

Moreover, for nearby canonicalized response distributions, this symmetric KL divergence provides a second-order local approximation to the squared Fisher--Rao distance.
\end{theorem}

\begin{proof}
For two multivariate complex Gaussian distributions with identical covariance $\Sigma$ and means $\mu_i,\mu_j$, the symmetric KL divergence is proportional to
\[
(\mu_i-\mu_j)^H
\Sigma^{-1}
(\mu_i-\mu_j).
\]
Applying this identity independently across the product distribution over admissible actions gives
\[
\sum_{a\in\cU_{\act}}
\Delta\widetilde{\mathbf h}_{a,ij}^{H}
\Sigma^{-1}
\Delta\widetilde{\mathbf h}_{a,ij},
\]
which is precisely Eq.~\eqref{eq:dsigma_scene} under action-independent covariance. The final statement follows from the standard local relationship between KL divergence and the Fisher information metric: for nearby distributions on a regular statistical manifold, KL-type divergences agree with the Fisher--Rao metric to second order.
\end{proof}

This result provides an information-geometric interpretation of
$d_{\Sigma,\mathrm{scene}}$
through the local geometry of statistical manifolds induced by canonicalized sensing distributions~\cite{amari2016information}. Under Gaussian noise, the metric measures statistical separation between canonicalized measurement distributions in a noise-weighted geometry induced by the sensing model. The quotient construction can therefore be viewed as identifying scene configurations that induce the same point on the canonicalized measurement-distribution manifold, while the distinguishability metric measures local statistical separation between quotient classes. This interpretation connects the proposed formulation to information geometry without requiring the learned latent space to be an exact isometric embedding of the quotient manifold.
\section{Scene--Nuisance Factorized OQ-TSAE}
\label{sec:scene_nuisance_factorized_oq_tsae}
\subsection{Design Rationale and Architectural Principles}
\label{subsec:design_rationale_and_architectural_principles}

The theoretical analysis of Section~\ref{sec:theoretical_analysis} establishes three requirements for a sensor-conditioned scene representation. First, the latent geometry should reflect the scene-relevant observation quotient, meaning that scene configurations that are observationally equivalent after nuisance canonicalization should be represented similarly, while reliably distinguishable scene configurations should remain separated. Second, the representation should suppress nuisance-driven variation arising from sensing hardware, calibration, and measurement non-idealities. Third, the representation should preserve sufficient information to support observable scene decoding and raw observation reconstruction.

To operationalize these requirements, we develop a quotient-oriented tensor autoencoding framework called Observation-Quotient Tucker-Structured Autoencoding (OQ-TSAE). The framework is designed around three architectural principles: (i) preserving shared multilinear sensing structure through a Tucker-based tensor encoder, (ii) separating scene-relevant and nuisance-dependent variation through factorized latent pathways, and (iii) shaping the scene representation through quotient-consistent supervision while retaining reconstruction compatibility. Figure~\ref{fig:oq_tsae_framework} summarizes the overall architecture and supervision structure before detailing each component.

\begin{figure*}
    \centering
    \includegraphics[width=\textwidth]
    {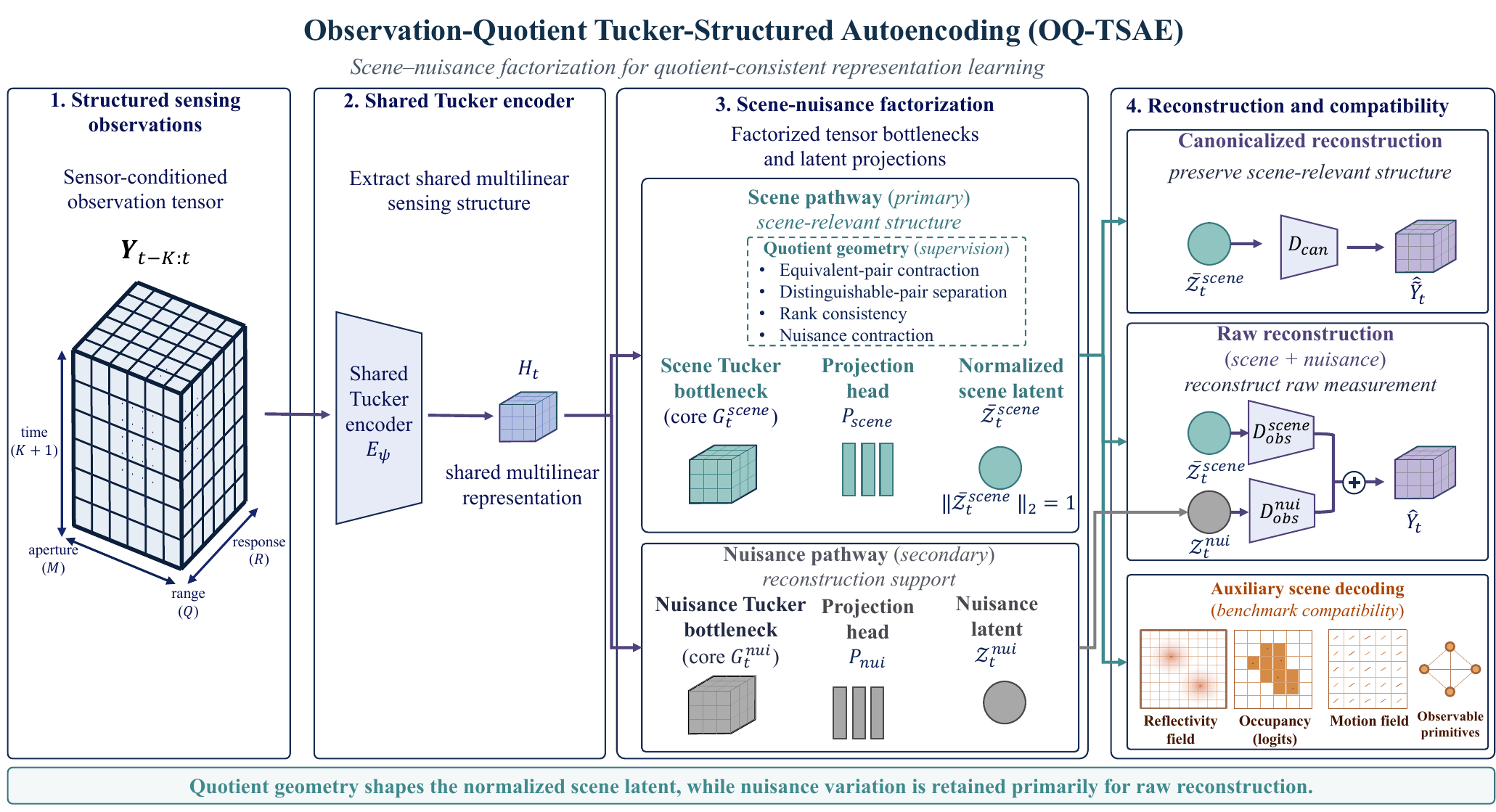}
    \caption{
    Overview of Observation-Quotient Tucker-Structured Autoencoding (OQ-TSAE). A shared Tucker encoder factorizes sensing observations into scene and nuisance pathways. Quotient geometry shapes the scene latent, whereas nuisance-dependent variation is retained primarily for raw observation reconstruction.
    }
    \label{fig:oq_tsae_framework}
\end{figure*}

Figure~\ref{fig:oq_tsae_framework} illustrates the central design logic of OQ-TSAE. Quotient geometry objectives shape the normalized scene latent, while raw observation reconstruction is supported through both scene and nuisance pathways and canonicalized reconstruction depends only on the scene pathway. Auxiliary scene decoders provide compatibility with observable scene interpretation.

These objectives are inherently in tension when using a single latent representation. A representation optimized primarily for reconstruction fidelity will naturally encode nuisance-driven measurement variation, since such variation contributes substantially to raw observations. Conversely, a representation forced to be nuisance-invariant may discard information needed to explain observation variability, degrading reconstruction quality and limiting compatibility with downstream sensing tasks.

To resolve this tension, OQ-TSAE adopts a scene--nuisance factorization strategy in which a shared tensor representation is decomposed into complementary scene and nuisance pathways. The scene pathway is explicitly constrained by quotient supervision derived from the scene-relevant distinguishability metric $d_{\Sigma,\mathrm{scene}}$, whereas the nuisance pathway retains additional representational capacity for hardware-dependent and calibration-related variation useful for reconstructing raw observations.

Importantly, the two pathways are not learned independently but share a common tensor encoder that extracts multilinear sensing structure before dedicated bottlenecks and projection heads produce separate scene and nuisance representations. This design allows quotient-consistent scene organization to coexist with reconstruction compatibility: the scene latent is encouraged to preserve sensing-supported scene distinctions, whereas nuisance-dependent measurement variation is retained primarily for raw observation reconstruction.

\subsection{Shared Tensor Encoder and Factorized Latent Pathways}
\label{subsec:shared_tensor_encoder_and_factorized_latent_pathwa}

The encoder $E_{\psi}$ takes as input the observation tensor
\[
\cY_{t-K:t}
\in
\mathbb C^{(K+1)\times M\times R\times Q},
\]
which aggregates temporal observations over a short sensing horizon. Rather than directly constructing separate scene and nuisance representations, the encoder first extracts a shared tensor representation that captures measurement structure jointly across temporal, sensing, response, and range modes.

Formally,
\begin{equation}
\cY_{t-K:t}
\xrightarrow{E_{\psi}}
\mathbf H_t,
\label{eq:shared_encoder}
\end{equation}
where $\mathbf H_t$ denotes an intermediate tensor representation produced by a multilinear encoder. The encoder preserves the tensor structure of the observation sequence through Tucker-style tensor operations and learned multilinear projections, allowing interactions across sensing modes to be modeled while maintaining a compact representation.

To separate scene-relevant structure from nuisance-driven variation, the shared tensor representation is subsequently processed through two dedicated Tucker bottlenecks:
\begin{equation}
\mathbf H_t
\rightarrow
(
\cG_t^{\scene},
\cG_t^{\nui}
),
\label{eq:dual_cores}
\end{equation}
where
\[
\cG_t^{\scene},
\cG_t^{\nui}
\in
\mathbb R^{
R_T\times R_M\times R_R\times R_Q
}
\]
denote the scene and nuisance tensor cores, respectively. These bottlenecks encourage scene-relevant and nuisance-related information to be routed into separate tensor pathways while preserving multilinear structure.

For convenience, we denote the concatenated tensor-core representation by

\[
\cG_t
=
[\cG_t^{\scene},\cG_t^{\nui}],
\]
which serves as a shared tensor-core feature representation for several auxiliary decoding heads introduced later in this section.

The tensor cores are subsequently flattened and mapped to latent representations through dedicated linear projection heads:
\begin{align}
z_t^{\scene}
&=
P_{\scene}
\!\left(
\operatorname{vec}
(
\cG_t^{\scene}
)
\right),
\label{eq:scene_latent}
\\
z_t^{\nui}
&=
P_{\nui}
\!\left(
\operatorname{vec}
(
\cG_t^{\nui}
)
\right),
\label{eq:nuis_latent}
\end{align}
where
$P_{\scene}$
and
$P_{\nui}$
are linear projection mappings producing the scene and nuisance latent codes.

The scene latent is explicitly normalized:
\begin{equation}
\bar z_t^{\scene}
=
\frac{
z_t^{\scene}
}{
\|z_t^{\scene}\|_2+\varepsilon
},
\label{eq:normalized_scene_latent}
\end{equation}
so that it lies on the unit sphere. This normalization serves an architectural purpose aligned with quotient supervision. Since quotient constraints operate on the normalized scene representation, scene organization must be encoded through the angular geometry of the latent space rather than through latent magnitude. Consequently, scene-equivalent configurations are encouraged to collapse geometrically, while reliably distinguishable scene variations remain separable through angular structure.

In contrast, the nuisance latent remains unnormalized, allowing greater flexibility for representing nuisance-related measurement variation useful for raw observation reconstruction.

\subsection{Tucker-Structured Observation Decoders}
\label{subsec:tucker_structured_observation_decoders}

The learned tensor representations are coupled to observation reconstruction through Tucker-structured decoders that preserve the multilinear organization of the sensing tensor. Rather than reconstructing observations directly from latent vectors through unrestricted fully connected mappings, reconstruction is performed from the factorized scene and nuisance tensor cores, allowing the decoder to exploit structured interactions across temporal, sensing, response, and range modes.

The raw observation reconstruction combines information from both scene and nuisance pathways:
\begin{equation}
\widehat{\cY}_t
=
D_{\obs}^{\scene}
(
\cG_t^{\scene}
)
+
D_{\obs}^{\nui}
(
\cG_t^{\nui}
),
\label{eq:raw_recon}
\end{equation}
where
$D_{\obs}^{\scene}$
and
$D_{\obs}^{\nui}$
denote Tucker-structured observation decoders operating on the scene and nuisance tensor cores, respectively. The scene pathway contributes scene-relevant measurement structure, while the nuisance pathway captures hardware-dependent and calibration-related variation that remains present in raw observations.

Canonicalized observation reconstruction depends only on the scene pathway:
\begin{equation}
\widehat{\widetilde{\cY}}_t
=
D_{\can}
(
\cG_t^{\scene}
),
\label{eq:can_recon}
\end{equation}
thereby removing nuisance-dependent variation by construction. Since nuisance canonicalization defines the quotient relation introduced in Section~\ref{sec:sensor_conditioned_scene_state_formulation}, reconstructing canonicalized observations from the scene tensor core encourages the learned representation to preserve scene-distinguishable measurement structure while suppressing nuisance-related variability.

Both observation decoders employ Tucker-structured multilinear transformations whose effective capacity is governed by the Tucker ranks
$(R_T,R_M,R_R,R_Q)$.
This tensor-structured design provides an inductive bias aligned with the multilinear organization of the sensing process while maintaining substantially lower effective complexity than unrestricted dense reconstruction.

Importantly, reconstruction is not treated as the primary objective of the proposed framework. Instead, reconstruction acts as a compatibility constraint ensuring that quotient-consistent scene organization remains sufficiently informative for observation recovery. Since reconstruction operates on tensor cores rather than directly on latent vectors, reconstruction gradients primarily supervise the factorized scene and nuisance pathways through the Tucker bottlenecks while quotient constraints continue to shape the geometry of the normalized scene latent.

\subsection{Auxiliary Scene Field and Primitive Decoders}
\label{subsec:auxiliary_scene_field_and_primitive_decoders}

Beyond observation reconstruction, the scene latent is connected to several auxiliary decoders that predict observable scene-level quantities for downstream interpretation. These heads are included to evaluate scene interpretability and downstream compatibility in the controlled benchmark; they are not required by the quotient formulation itself and should not be interpreted as defining the dimensionality of the quotient space.

The primary scene-field decoder maps the normalized scene latent to a canonicalized reflectivity field at a sensor-compatible spatial resolution:
\begin{equation}
\widehat{\widetilde{\boldsymbol{\rho}}}_{t}^{\res,\scene}
=
D_{\rho}^{\scene}
(
\bar z_t^{\scene}
)
\in
\mathbb C^{
N_x^{\res}
\times
N_y^{\res}
\times
N_z^{\res}
},
\label{eq:dec_rho_scene}
\end{equation}
which predicts the coarse scene-level reflectivity structure over the observable spatial grid in the controlled benchmark.

To improve local reconstruction quality, an auxiliary residual reflectivity decoder predicts a correction field:
\begin{equation}
\widehat{\widetilde{\boldsymbol{\rho}}}_{t}^{\res,\mathrm{res}}
=
D_{\rho}^{\mathrm{res}}
(
\operatorname{vec}(\cG_t)_{\mathrm{detach}}
),
\label{eq:dec_rho_res}
\end{equation}
where the residual pathway operates on detached tensor-core features to refine reflectivity prediction without directly altering the quotient-constrained latent geometry. The final reflectivity prediction is obtained through additive composition:
\begin{equation}
\widehat{\widetilde{\boldsymbol{\rho}}}_{t}^{\res}
=
\widehat{\widetilde{\boldsymbol{\rho}}}_{t}^{\res,\scene}
+
\alpha_{\rho}
\widehat{\widetilde{\boldsymbol{\rho}}}_{t}^{\res,\mathrm{res}},
\label{eq:dec_rho_final}
\end{equation}
where $\alpha_{\rho}$ is a fixed residual weighting coefficient.

In addition, the benchmark-instantiated motion decoder predicts a scene-level planar motion field:
\begin{equation}
\widehat{\boldsymbol{\upsilon}}_{t}^{\res}
=
D_{v}
(
\bar z_t^{\scene}
)
\in
\mathbb R^{
2
\times
N_x^{\res}
\times
N_y^{\res}
\times
N_z^{\res}
},
\label{eq:dec_v}
\end{equation}
providing a spatially resolved estimate of the observable Doppler-related planar motion dynamics used in the controlled benchmark.

An auxiliary occupancy head predicts grid-cell occupancy logits from the scene latent:
\begin{equation}
\widehat{\mathbf o}_t
=
D_{\mathrm{occ}}
(
\bar z_t^{\scene}
),
\label{eq:occ_head}
\end{equation}
which provides additional supervision for sparse scene localization over the observable spatial grid.

The model additionally predicts a lightweight primitive-level scene representation consisting of object presence, position, and reflectivity attributes:
\begin{equation}
\widehat{\mathbf p}_t
=
D_{\mathrm{prim}}
(
\bar z_t^{\scene}
),
\label{eq:primitive_head}
\end{equation}
serving as an auxiliary structured scene representation compatible with downstream object-level interpretation. A primitive-rendered reflectivity field may further be derived from these predictions during training as an auxiliary regularization signal.

Several interpretable outputs are obtained through lightweight post-processing of these decoded quantities. Occupancy estimates are derived from occupancy logits or reflectivity magnitude thresholding, target locations are obtained through peak detection on the reflectivity field, and motion estimates are extracted from the decoded motion field at detected target locations.

Importantly, these auxiliary decoders do not define the quotient geometry itself. Rather, they serve as evaluation- and compatibility-oriented heads used to assess whether quotient-consistent scene organization preserves sufficient information for observable scene interpretation after nuisance-related variation has been suppressed.

\section{Learning Objectives and Evaluation Metrics}
\label{sec:learning_objectives_and_evaluation_metrics}

\subsection{Overview of the Training Objectives}
\label{subsec:overview_of_the_training_objectives}

The training objective of OQ-TSAE combines three categories of losses: reconstruction objectives, quotient-geometry objectives, and auxiliary scene-output objectives.

Reconstruction objectives preserve compatibility with raw and canonicalized observations. Quotient-geometry objectives operate on the normalized scene latent and provide the primary learning signal through equivalent-pair contraction, distinguishable-pair separation, rank-consistency supervision, and nuisance-only invariance. Auxiliary scene-output objectives provide additional supervision for observable scene interpretation, including reflectivity, occupancy, motion, and primitive-level prediction.

Among these components, the quotient-geometry objectives
($\mathcal L_{\eq}$,
$\mathcal L_{\dist}$,
$\mathcal L_{\mathrm{rank}}$,
and
$\mathcal L_{\inv}$)
define the primary representation objective, whereas reconstruction and auxiliary scene-output losses serve mainly as compatibility constraints that preserve observation fidelity and observable scene interpretability. Accordingly, representation correctness is evaluated primarily through the quotient-oriented diagnostics introduced later in this section rather than through reconstruction metrics alone.

The following subsections describe each objective in detail.

\subsection{Reconstruction and Observation Compatibility Objectives}
\label{subsec:reconstruction_and_observation_compatibility_objec}

To maintain reconstruction compatibility, OQ-TSAE employs auxiliary objectives on both raw and canonicalized observations.

The raw observation reconstruction objective combines information from the scene and nuisance pathways:
\begin{equation}
\mathcal{L}_{\raw}
=
\big\|
\widehat{\cY}_t
-
\cY_t
\big\|_F^2,
\label{eq:L_raw}
\end{equation}
where
\[
\widehat{\cY}_t
=
D_{\obs}^{\scene}
(
\cG_t^{\scene}
)
+
D_{\obs}^{\nui}
(
\cG_t^{\nui}
).
\]
This objective encourages the representation to preserve information necessary for reconstructing raw observations, including nuisance-dependent measurement variation.

The canonicalized observation reconstruction objective depends only on the scene pathway:
\begin{equation}
\mathcal{L}_{\can}
=
\big\|
\widehat{\widetilde{\cY}}_t
-
\widetilde{\cY}_t
\big\|_F^2,
\label{eq:L_can}
\end{equation}
where
\[
\widehat{\widetilde{\cY}}_t
=
D_{\can}
(
\cG_t^{\scene}
),
\]
and
\[
\widetilde{\cY}_t
=
\calop(\cY_t)
\]
denotes the nuisance-canonicalized observation tensor. This objective encourages the scene pathway to preserve scene-relevant measurement structure after nuisance removal.

Masked observation reconstruction is further employed as an auxiliary regularizer:
\begin{equation}
\mathcal{L}_{\mathrm{mask}}
=
\sum_{u\in\mathcal M_{\raw}}
\|
\widehat y(u)-y(u)
\|_2^2
+
\sum_{u\in\mathcal M_{\can}}
\|
\widehat{\widetilde y}(u)
-
\widetilde y(u)
\|_2^2,
\label{eq:L_mask}
\end{equation}
where masked entries are reconstructed from partially observed raw and canonicalized tensors. This objective encourages exploitation of cross-mode structure under partial observation.

\subsection{Quotient-Geometry Objectives}
\label{subsec:quotient_geometry_objectives}

The quotient-geometry objectives provide the primary learning signal connecting the scene-relevant observation quotient to the learned scene representation. Although related in form to metric-learning objectives~\cite{hadsell2006dimensionality}, they are used here to enforce quotient-consistent scene organization induced by sensing-supported distinguishability. All quotient constraints operate on the normalized scene latent $\bar z_t^{\scene}$.

For reliably equivalent scene pairs $(i,j)\in\mathcal P_{\eq}^{\scene}$, the contraction objective is defined as
\begin{equation}
\mathcal L_{\eq}
=
\frac{1}{|\mathcal P_{\eq}^{\scene}|}
\sum_{(i,j)\in\mathcal P_{\eq}^{\scene}}
\left\|
\bar z_i^{\scene}
-
\bar z_j^{\scene}
\right\|_2^2 .
\label{eq:L_eq}
\end{equation}
This loss contracts scene pairs that remain indistinguishable after nuisance canonicalization.

For reliably distinguishable scene pairs $(i,j)\in\mathcal P_{\dist}^{\scene}$, the separation objective is
\begin{equation}
\mathcal L_{\dist} =
\frac{1}{|\mathcal P_{\dist}^{\scene}|}
\sum_{(i,j)\in\mathcal P_{\dist}^{\scene}}
\left[
m_{\dist}
-
\left\|
\bar z_i^{\scene}
-
\bar z_j^{\scene}
\right\|_2
\right]_+^2 ,
\label{eq:L_sep}
\end{equation}
where $m_{\dist}>0$ is the latent separation margin and $[x]_+=\max(x,0)$. This loss preserves separation between reliably distinguishable scene configurations.

To preserve coarse distinguishability ordering, we further use a squared-hinge rank-consistency objective over reliably ordered triplets $(i,j,k)\in\mathcal T^{\scene}$:
\begin{equation}
\begin{aligned}
\mathcal L_{\mathrm{rank}}
&=
\frac{1}{|\mathcal T^{\scene}|}
\sum_{(i,j,k)\in\mathcal T^{\scene}}
\Bigl[
m_{\mathrm{rank}}
\\
&\quad+
\left\|
\bar z_i^{\scene}
-\bar z_j^{\scene}
\right\|_2
-
\left\|
\bar z_i^{\scene}
-\bar z_k^{\scene}
\right\|_2
\Bigr]_+^{2}.
\end{aligned}
\label{eq:L_rank}
\end{equation}

Here $\mathcal T^{\scene}$ contains triplets satisfying
\[
d_{\Sigma,\mathrm{scene}}(s_i,s_j)
+
\epsilon_{\mathrm{gap}}
<
d_{\Sigma,\mathrm{scene}}(s_i,s_k),
\]
where $\epsilon_{\mathrm{gap}}>0$ excludes near-tied distinguishability comparisons and is separate from the equivalence and distinguishability thresholds used for pair supervision. The objective preserves reliable scene-distinguishability ordering without requiring metric calibration to $d_{\Sigma,\mathrm{scene}}$.

Finally, nuisance-only contraction supervision is applied to nuisance-only pairs $(i,j)\in\mathcal P_{\nui}$:
\begin{equation}
\mathcal L_{\inv}
=
\frac{1}{|\mathcal P_{\nui}|}
\sum_{(i,j)\in\mathcal P_{\nui}}
\left\|
\bar z_i^{\scene}
-
\bar z_j^{\scene}
\right\|_2^2 .
\label{eq:L_inv}
\end{equation}
This objective suppresses residual nuisance leakage that may remain after nuisance canonicalization due to imperfect canonicalization or finite-sample estimation, while preserving sensing-supported scene distinctions by acting only on nuisance-only pairs.

\subsection{Auxiliary Scene Output Supervision}
\label{subsec:auxiliary_scene_output_supervision}

Beyond the quotient-geometry objectives, OQ-TSAE employs auxiliary supervision on observable scene quantities to maintain compatibility with downstream scene interpretation.

The reflectivity-field objective supervises prediction of the canonicalized scene reflectivity field:
\begin{equation}
\mathcal{L}_{\rho}
=
\mathcal{L}_{\mathrm{mag}}
+
\lambda_{\mathrm{occ}}
\mathcal{L}_{\mathrm{occ}}
+
\lambda_{\mathrm{bg}}
\mathcal{L}_{\mathrm{bg}}
+
\lambda_{\mathrm{phase}}
\mathcal{L}_{\mathrm{phase}}
+
\lambda_{\mathrm{sp}}
\mathcal{L}_{\mathrm{sp}},
\label{eq:L_rho}
\end{equation}
where
$\mathcal{L}_{\mathrm{mag}}$
encourages reflectivity-magnitude consistency,
$\mathcal{L}_{\mathrm{occ}}$
emphasizes occupied spatial cells,
$\mathcal{L}_{\mathrm{bg}}$
penalizes background reconstruction error,
$\mathcal{L}_{\mathrm{phase}}$
encourages phase consistency, and
$\mathcal{L}_{\mathrm{sp}}$
penalizes diffuse activations in unoccupied regions. Together, these terms encourage sparse and localized scene-field reconstruction.

The occupancy-logit head is trained using
\begin{equation}
\mathcal{L}_{\mathrm{occ\text{-}logit}}
=
\mathcal{L}_{\mathrm{BCE}}
+
\lambda_{\mathrm{dice}}
\mathcal{L}_{\mathrm{dice}},
\label{eq:L_occ_logit}
\end{equation}
which provides supervision for sparse occupancy prediction over the observable spatial grid.

When motion supervision is available, the motion-field decoder is trained using
\begin{equation}
\mathcal{L}_{v}
=
\|
\widehat{\boldsymbol{\upsilon}}_t^{\res}
-
\boldsymbol{\upsilon}_t^{\res}
\|_2^2,
\label{eq:L_v}
\end{equation}
providing an auxiliary signal for observable motion recovery.

Primitive-level supervision is applied through
\begin{equation}
\mathcal{L}_{\mathrm{prim}}
=
\mathcal{L}_{\mathrm{presence}}
+
\lambda_{\mathrm{pos}}
\mathcal{L}_{\mathrm{pos}}
+
\lambda_{\mathrm{refl}}
\mathcal{L}_{\mathrm{refl}},
\label{eq:L_prim}
\end{equation}
where
$\mathcal{L}_{\mathrm{presence}}$
supervises object presence,
$\mathcal{L}_{\mathrm{pos}}$
penalizes matched position error, and
$\mathcal{L}_{\mathrm{refl}}$
supervises primitive reflectivity attributes. A primitive-rendered reflectivity loss is implemented as an optional auxiliary regularizer but is assigned zero weight in the default experimental configuration.

These objectives do not define the quotient geometry. Their role is to maintain compatibility with observable scene decoding, occupancy localization, motion estimation, and primitive-level scene interpretation.

\subsection{Unified Training Objective}
\label{subsec:unified_training_objective}

The complete training objective combines reconstruction compatibility, quotient geometry, and auxiliary scene-output supervision:

\begin{equation}
\begin{aligned}
\mathcal{L}
=
&
\;
\lambda_{\raw}\mathcal{L}_{\raw}
+
\lambda_{\can}\mathcal{L}_{\can}
+
\lambda_{\mathrm{mask}}
\mathcal{L}_{\mathrm{mask}}
\\
&
+
\lambda_{\eq}\mathcal{L}_{\eq}
+
\lambda_{\dist}\mathcal{L}_{\dist}
+
\lambda_{\mathrm{rank}}
\mathcal{L}_{\mathrm{rank}}
+
\lambda_{\inv}\mathcal{L}_{\inv}
\\
&
+
\lambda_{\rho}\mathcal{L}_{\rho}
+
\lambda_{\mathrm{occ\text{-}logit}}
\mathcal{L}_{\mathrm{occ\text{-}logit}}
+
\lambda_{v}\mathcal{L}_{v}
+
\lambda_{\mathrm{prim}}
\mathcal{L}_{\mathrm{prim}}
\\
&
+
\lambda_{\mathrm{prim\text{-}render}}
\mathcal{L}_{\mathrm{prim\text{-}render}},
\end{aligned}
\label{eq:full_loss}
\end{equation}
where the loss weights $\lambda_{\cdot}$ balance the different objective components. The primitive-rendered reflectivity term is included for completeness but assigned zero weight in the default experimental configuration.

The losses act on different components of the factorized architecture. The quotient-geometry objectives
(
$\mathcal{L}_{\eq}$,
$\mathcal{L}_{\dist}$,
$\mathcal{L}_{\mathrm{rank}}$,
$\mathcal{L}_{\inv}$
)
operate on the normalized scene latent
$\bar z^{\scene}$ and directly shape the quotient geometry.

By contrast, observation reconstruction losses operate on the tensor-core pathways.
The canonicalized reconstruction objective
$\mathcal{L}_{\can}$
is applied to the scene reconstruction component decoded from
$\cG^{\scene}$,
whereas the raw reconstruction objective
$\mathcal{L}_{\raw}$
is applied to the combined scene and nuisance reconstruction decoded from
$\cG^{\scene}$
and
$\cG^{\nui}$.

The auxiliary scene-output losses
(
$\mathcal{L}_{\rho}$,
$\mathcal{L}_{\mathrm{occ\text{-}logit}}$,
$\mathcal{L}_{v}$,
$\mathcal{L}_{\mathrm{prim}}$
)
provide compatibility supervision for observable scene decoding.

Overall, quotient consistency serves as the primary representation objective, while reconstruction and auxiliary scene-output losses maintain compatibility with observation recovery and scene interpretation.

\subsection{Diagnostic Metrics}
\label{subsec:diagnostic_metrics}

The following diagnostics provide empirical measurements of the quotient-consistency properties analyzed in Section~\ref{sec:theoretical_analysis}. Together, they evaluate whether the learned scene latent contracts quotient-equivalent scenes, preserves scene-distinguishable differences, suppresses nuisance variation, and maintains distinguishability ordering.

The False Distinction Rate (FDR) measures the fraction of reliably equivalent scene pairs whose latent codes are incorrectly separated beyond a threshold $\delta$:

\begin{equation}
\operatorname{FDR}(\delta)
=
\frac{
\left|
\{
(i,j)\in\mathcal P_{\eq}^{\scene}
:
\|
\bar z_i^{\scene}
-
\bar z_j^{\scene}
\|_2
>
\delta
\}
\right|
}{
|
\mathcal P_{\eq}^{\scene}
|
}.
\label{eq:fdr}
\end{equation}

Low FDR indicates successful contraction of quotient-equivalent scene pairs.

The False Merge Rate (FMR) measures the fraction of reliably distinguishable scene pairs whose latent codes are incorrectly collapsed within the same threshold:

\begin{equation}
\operatorname{FMR}(\delta)
=
\frac{
\left|
\{
(i,j)\in\mathcal P_{\dist}^{\scene}
:
\|
\bar z_i^{\scene}
-
\bar z_j^{\scene}
\|_2
\le
\delta
\}
\right|
}{
|
\mathcal P_{\dist}^{\scene}
|
}.
\label{eq:fmr}
\end{equation}

Low FMR indicates successful preservation of scene-distinguishable pairs.

The Nuisance Sensitivity Rate (NSR) is defined on nuisance-only scene pairs sharing the same scene configuration but differing in nuisance realization:

\begin{equation}
\operatorname{NSR}(\delta)
=
\frac{
\left|
\{
(i,j)\in\mathcal P_{\nui}
:
\|
\bar z_i^{\scene}
-
\bar z_j^{\scene}
\|_2
>
\delta
\}
\right|
}{
|
\mathcal P_{\nui}
|
}.
\label{eq:nsr}
\end{equation}

Low NSR indicates effective suppression of nuisance-dependent variation in the scene latent.

Finally, latent ordering consistency is evaluated using a rank consistency score that measures whether the learned geometry preserves pairwise distinguishability ordering induced by the scene-relevant metric. Let

\[
(i,j,k)
\]

denote a scene triplet satisfying

\[
d_{\Sigma,\mathrm{scene}}(s_i,s_j)
+
\epsilon_{\mathrm{gap}}
<
d_{\Sigma,\mathrm{scene}}(s_i,s_k),
\]

where $\epsilon_{\mathrm{gap}}$ excludes near-tied distinguishability comparisons. The rank consistency score is defined as

\begin{equation}
\begin{aligned}
\operatorname{Rank}
=
\frac{1}{|\mathcal T^{\scene}|}
\left|
\left\{
(i,j,k)\in\mathcal T^{\scene} :
\right.\right.
\\[-1mm]
\left.\left.
\|\bar z_i^{\scene}-\bar z_j^{\scene}\|_2
<
\|\bar z_i^{\scene}-\bar z_k^{\scene}\|_2
\right\}
\right|,
\end{aligned}
\label{eq:rank}
\end{equation}

High rank consistency indicates that the learned geometry preserves the coarse ordering induced by scene-relevant distinguishability without requiring exact metric calibration.

These metrics should be interpreted jointly because each captures a distinct failure mode. Low FDR with high NSR indicates nuisance leakage, whereas low FDR with high FMR indicates excessive latent collapse. Conversely, high FDR with low FMR indicates over-separation, often driven by nuisance-sensitive variation rather than scene-relevant distinctions. A quotient-consistent representation should therefore achieve low FDR, low FMR, low NSR, and high rank consistency. The mixed scene-pair families introduced in Section~\ref{subsec:scene_pair_taxonomy} provide an additional stress test distinguishing genuine scene--nuisance factorization from trivial invariance strategies.

Together, these diagnostics operationalize representation correctness by measuring false distinctions, false merges, nuisance leakage, and distinguishability-order preservation.

\section{Controlled Sensor-Conditioned Benchmark}
\label{sec:controlled_sensor_conditioned_benchmark}

\subsection{Benchmark Design Principles}
\label{subsec:benchmark_design_principles}

The benchmark is designed to evaluate representation correctness under structured sensing, namely whether a learned representation preserves sensing-supported scene distinctions while suppressing nuisance-driven and sensor-indistinguishable variation. Throughout the benchmark, scene-relevant labels are derived from distinguishability computed after nuisance canonicalization, so that representation correctness is evaluated with respect to scene structure rather than hardware state. Unlike conventional perception benchmarks that emphasize reconstruction or localization accuracy, the present benchmark evaluates whether latent geometry remains consistent with scene-relevant distinguishability.

Controlled simulation is used as a mechanism-analysis protocol rather than as a deployment benchmark. The objective is not to reproduce the full complexity of real sensing systems, but to isolate the factors relevant to quotient-consistent representation learning, including scene perturbation, nuisance perturbation, sensing-supported distinguishability, and nuisance canonicalization. Such factors are difficult to disentangle in unconstrained sensing datasets, where scene change, calibration drift, clutter, multipath, and annotation uncertainty are often entangled. Accordingly, the benchmark combines controlled scene-pair families, sensor-compatible observable supervision, and sensing-dependent distinguishability analysis to evaluate whether latent geometry tracks scene-relevant distinguishability rather than raw measurement variation or physical displacement alone~\cite{oberkampf2010verification}.
The resulting benchmark serves as the primary mechanism-analysis platform throughout the paper. A complementary real-radar utility validation on the CARRADA dataset~\cite{ouaknine2021carrada} is presented later to assess downstream utility beyond the controlled benchmark setting.

\subsection{Primary Sensing Platform}
\label{subsec:primary_sensing_platform}

We adopt a sensor-conditioned tensor sensing simulator as the primary evaluation platform because it provides controllable sensing diversity, resolution, and nuisance variation under an analytically tractable forward model. The simulator serves as a controlled benchmark instantiation of the proposed framework rather than as a model of any specific sensing modality.

The simulator models a structured sensing process that generates observations over controllable sensing, response, and range dimensions. Controllable sensing dimensions include sensing-state selection
$m \in \{1,\ldots,M\}$
and response-mode selection
$r \in \{1,\ldots,R\}$,
with observations represented over
$Q$
range bins. The forward model follows Eq.~\eqref{eq:forward_model}, with platform-specific response functions determined by the sensing configuration.

Unless otherwise specified, experiments use the primary sensing configuration with
$M = 16$ sensing states,
$R = 4$ response modes,
$Q = 32$ range bins,
signal-to-noise ratio of $20$ dB,
and temporal observation length
$T = 8$.

\subsection{Scene-Relevant Quotient Supervision}
\label{subsec:scene_relevant_quotient_supervision}

Observable scene targets are defined at a sensor-compatible resolution so that supervision remains consistent with the sensing capability of the platform. The observable reflectivity field is represented on a discretized spatial grid that provides stable reconstruction supervision while remaining computationally tractable. Motion supervision, when included, is defined over the same observable spatial grid and temporal observation horizon. These observable targets serve as auxiliary supervision for scene interpretation rather than defining the quotient geometry itself.

The quotient relation is defined through the scene-relevant distinguishability metric $d_{\Sigma,\mathrm{scene}}$ computed after nuisance canonicalization, as formalized in Sections~\ref{subsec:scene_relevant_observation_quotient} and~\ref{subsec:scene_relevant_distinguishability_metric}. In the present study, experiments employ an oracle supervision regime in which $d_{\Sigma,\mathrm{scene}}$ is computed directly from the simulator forward model after nuisance canonicalization. This design isolates representation behavior from quotient-estimation error and should therefore be interpreted as a mechanism-identification setting rather than a deployment assumption. The effects of supervision quality, supervision mismatch, and quotient-relation perturbations are examined empirically in Section~\ref{subsec:mechanistic_validation}.

\subsection{Scene-Pair Taxonomy}
\label{subsec:scene_pair_taxonomy}

Representation correctness is evaluated through a taxonomy of scene-pair perturbations spanning geometry, reflectivity, nuisance, and mixed scene--nuisance variation. The taxonomy operationalizes the merge-versus-separate decisions implied by the scene-relevant observation quotient by distinguishing perturbations that remain below the scene-relevant distinguishability threshold from those that become reliably distinguishable under the sensing process.

The resulting benchmark evaluates whether a representation contracts quotient-equivalent variation while preserving sensing-supported scene distinctions and suppressing nuisance-induced observation differences. Motion supervision is used only for auxiliary scene decoding and is therefore not treated as a separate perturbation family. The complete scene-pair taxonomy and diagnostic targets are provided in Appendix~\ref{app:scene_pair_taxonomy}.

\subsection{Controlled Benchmark Protocol}
\label{sec:controlled_benchmark_protocol}

The benchmark constructs controlled training and evaluation pairs using the scene-pair taxonomy described above. Each family is generated by applying a specified scene or nuisance perturbation to a base scene and assigning the corresponding merge-or-separate target according to $d_{\Sigma,\mathrm{scene}}$. The resulting protocol directly evaluates quotient-consistent representation behavior through the representation-correctness diagnostics introduced in Section~\ref{subsec:diagnostic_metrics}.

In addition to the scene-pair families, the benchmark includes a scene-level distinguishability protocol that examines how sensing capability influences representation resolution. Single-target scene pairs with controlled spatial separation $\Delta p$ are evaluated under sensing configurations of varying capability. For each configuration, we compute both $d_{\Sigma,\mathrm{scene}}$ and the corresponding latent distance
$\norm{\bar{z}_i^{\scene} - \bar{z}_j^{\scene}}$.
A companion analysis additionally evaluates raw measurement distinguishability for comparison. Together, these protocols assess whether latent organization follows scene-relevant distinguishability rather than raw measurement variation alone.

\subsection{Evaluation Diagnostics}
\label{subsec:evaluation_diagnostics}

Representation correctness is evaluated using the four quotient-oriented diagnostics introduced in Section~\ref{subsec:diagnostic_metrics}: FDR, FMR, NSR, and rank consistency. These diagnostics are reported jointly because no single metric is sufficient to characterize representation correctness. Together, they quantify the principal failure modes of quotient-consistent representation learning, including false distinctions, false merges, nuisance leakage, and loss of distinguishability structure.

The benchmark therefore provides a controlled setting for evaluating whether a learned representation preserves scene-relevant distinguishability while suppressing nuisance-driven variation. We next use this benchmark to examine representation correctness, latent mechanism behavior, supervision robustness, and sensing-dependent representation organization.

\section{Experimental Evaluation}
\label{sec:experimental_evaluation}
\subsection{Experimental Configuration}
\label{subsec:experimental_configuration}

The primary experiments were conducted on the controlled sensor-conditioned benchmark described in Section~\ref{sec:controlled_sensor_conditioned_benchmark}. Unless otherwise specified, all methods used the same sensing protocol, scene-pair construction procedure, optimization budget, and evaluation diagnostics. The primary sensing configuration defined in Section~\ref{subsec:primary_sensing_platform} was used throughout the controlled-benchmark experiments.

To complement the controlled benchmark, we additionally evaluate learned representations on the CARRADA real-radar dataset~\cite{ouaknine2021carrada}, as described in Section~\ref{subsec:real_radar_validation}. Because CARRADA does not provide oracle quotient supervision, the real-radar experiment serves as a downstream utility validation rather than a direct quotient-correctness evaluation.

For the controlled-benchmark experiments, the proposed OQ-TSAE employed Tucker factorization ranks of $R_T = 4$, $R_M = 8$, $R_R = 4$, and $R_Q = 8$ for both the scene and nuisance pathways, with latent dimensions of $64$ per pathway and hidden dimension $256$. Training used the scene-pair taxonomy introduced in Section~\ref{subsec:scene_pair_taxonomy}, with quotient supervision derived from the oracle scene-relevant distinguishability metric $d_{\Sigma,\mathrm{scene}}$. The oracle metric was used only for reliable supervision-pair construction and was not provided during inference.

Optimization used Adam with learning rate $10^{-4}$, batch size $64$, and $100$ training epochs. Training scenes were generated using the benchmark family-construction protocol with $50$ samples per scene-pair family. Unless otherwise specified, results on the controlled benchmark were averaged over five random seeds $\{101,202,303,404,505\}$. Seed-to-seed variability remained small and did not alter the qualitative ordering of the reported diagnostics. Architectural and optimization settings were fixed across all methods and sensing configurations.

We compared OQ-TSAE against a set of representation baselines designed to isolate three factors central to the proposed framework: quotient supervision, tensorized observation modeling, and scene--nuisance factorization. Where applicable, latent dimensionality and optimization settings were matched to OQ-TSAE to reduce confounding effects of model scale.

The baseline set consists of FlatMLP, TuckerAE, TuckerAE-Field, MetricReg-raw, OQ-TSAE$_{\mathrm{raw}}$, MetricLearning-scene, and ContrastiveScene. FlatMLP processes observations using a multilayer perceptron without tensor structure or quotient supervision. TuckerAE performs Tucker-structured reconstruction without quotient supervision or nuisance separation, while TuckerAE-Field additionally incorporates observable scene-field decoding. MetricReg-raw replaces quotient supervision with pairwise latent-distance regression toward raw measurement distinguishability. OQ-TSAE$_{\mathrm{raw}}$ preserves the same architecture as OQ-TSAE$_{\mathrm{scene}}$ but replaces scene-relevant distinguishability with raw measurement distinguishability for supervision-pair construction.

MetricLearning-scene and ContrastiveScene use the same scene-relevant supervision source as OQ-TSAE$_{\mathrm{scene}}$ while excluding tensorized observation modeling, scene--nuisance factorization, and auxiliary reconstruction. The metric-learning baseline optimizes a Euclidean triplet-margin objective, whereas the contrastive baseline employs an InfoNCE-style objective.

Detailed optimization hyperparameters, supervision thresholds, margins, and loss weights are reported in Appendix~\ref{app:hyperparameters}.

\subsection{Representation Correctness}
\label{subsec:representation_correctness}

This experiment evaluates representation correctness using the quotient-oriented diagnostics introduced in Section~\ref{subsec:diagnostic_metrics}, namely FDR, FMR, NSR, and Rank.

\begin{table}
\centering
\caption{
Representation-correctness diagnostics at $\delta=0.3$.
(a) Aggregate comparison across methods.
(b) Per-family decomposition for selected scene-aware methods.
}
\label{tab:correctness}

\small
\setlength{\tabcolsep}{4pt}
\textbf{(a) Aggregate diagnostics}

\vspace{0.5mm}
\begin{tabular}{lcccc}
\toprule
\textbf{Method}
& \textbf{FDR}$\downarrow$
& \textbf{FMR}$\downarrow$
& \textbf{NSR}$\downarrow$
& \textbf{Rank}$\uparrow$ \\
\midrule

FlatMLP
& 0.166
& 0.001
& 0.998
& 0.634 \\

TuckerAE
& \textbf{0.000}
& 0.822
& \textbf{0.100}
& 0.631 \\

TuckerAE-Field
& 0.026
& 0.069
& 0.582
& \underline{0.753} \\

MetricReg-raw
& 0.076
& 0.007
& 0.976
& 0.651 \\
\midrule

MetricLearning-scene
& \underline{0.002}
& \textbf{0.000}
& 1.000
& 0.654 \\

ContrastiveScene
& \textbf{0.000}
& \underline{0.002}
& 0.986
& 0.659 \\

OQ-TSAE$_{\mathrm{raw}}$
& 0.020
& 0.004
& 0.994
& 0.717 \\
\midrule

\textbf{OQ-TSAE$_{\mathrm{scene}}$}
& 0.008
& \textbf{0.000}
& \underline{0.106}
& \textbf{0.838} \\

\bottomrule
\end{tabular}

\vspace{2mm}

\footnotesize
\setlength{\tabcolsep}{3pt}
\textbf{(b) Per-family diagnostic decomposition}

\vspace{0.5mm}
\begin{tabular}{lccc}
\toprule
\textbf{Family}
& \textbf{TuckerAE-Field}
& \textbf{OQ-TSAE$_{\mathrm{raw}}$}
& \textbf{OQ-TSAE$_{\mathrm{scene}}$} \\
\midrule
G-Eq
& FDR: 0.052
& FDR: 0.040
& FDR: 0.016 \\

G-Dist
& FMR: 0.000
& FMR: 0.000
& FMR: 0.000 \\

R-Eq
& FDR: 0.000
& FDR: 0.000
& FDR: 0.000 \\

R-Dist
& FMR: 0.156
& FMR: 0.016
& FMR: 0.000 \\

N-Gain
& NSR: 0.448
& NSR: 0.988
& NSR: 0.104 \\

N-Phase
& NSR: 0.716
& NSR: 1.000
& NSR: 0.108 \\
\midrule
G+Gain
& FMR: 0.000
& FMR: 0.000
& FMR: 0.000 \\

R+Phase
& FMR: 0.120
& FMR: 0.000
& FMR: 0.000 \\
\bottomrule
\end{tabular}
\end{table}

Table~\ref{tab:correctness}(a) reports the aggregate representation-correctness diagnostics at the operating threshold $\delta=0.3$.

The baseline methods exhibit distinct representation failure patterns. FlatMLP and MetricReg-raw achieve low FMRs, indicating that many scene-distinguishable pairs remain separated. However, both methods exhibit NSR close to unity, suggesting that hardware-induced measurement variation remains strongly encoded in the learned representation.

TuckerAE exhibits a different failure mode. Although it achieves near-zero FDR and relatively low NSR, its FMR remains extremely high ($0.822$), indicating substantial collapse of scene-distinguishable pairs. TuckerAE-Field partially alleviates this issue through observable scene supervision, reducing false merges and improving rank consistency, but still retains considerable nuisance sensitivity.

MetricLearning-scene and ContrastiveScene provide controls for pairwise supervision alone. Despite optimizing different objectives (triplet-margin versus InfoNCE), both methods achieve near-zero FDR and FMR while retaining nearly maximal nuisance sensitivity (NSR $\approx 1$). Their similar behavior suggests that the choice of pairwise objective is not the primary factor governing representation correctness. Although both methods separate scene-distinguishable pairs, nuisance-visible variation remains strongly encoded in the latent geometry.

Among the evaluated methods, OQ-TSAE$_{\mathrm{scene}}$ exhibits the strongest overall representation-correctness profile under the evaluated setting. It simultaneously achieves low FDR, near-zero FMR, low NSR, and the highest rank consistency. These results are consistent with the interpretation that the learned scene latent contracts quotient-equivalent scene pairs, preserves reliably distinguishable scene differences, suppresses nuisance-driven variation, and maintains coherent latent ordering.

\paragraph{Why Scene-Relevant Quotient Supervision Matters.}

To isolate the contribution of quotient supervision, we compare OQ-TSAE$_{\mathrm{scene}}$ and OQ-TSAE$_{\mathrm{raw}}$, which share the same architecture, optimization procedure, and latent organization while differing only in the supervision signal used to construct reliable scene pairs.

Table~\ref{tab:correctness}(a) shows that the supervision signal substantially affects representation behavior. Although both variants achieve low FDR and FMR, OQ-TSAE$_{\mathrm{raw}}$ exhibits NSR close to unity, indicating substantial nuisance leakage into the learned scene latent. In contrast, OQ-TSAE$_{\mathrm{scene}}$ substantially reduces NSR while improving rank consistency, suggesting stronger alignment between latent geometry and scene-relevant distinguishability.

This difference is consistent with how reliable scene pairs are constructed. Under raw distinguishability supervision, nuisance-visible perturbations directly contribute to pair separation because any measurement-visible change may be treated as evidence of scene difference. In contrast, scene-relevant quotient supervision is constructed after nuisance canonicalization and therefore contracts nuisance-equivalent observations while preserving separation for scene-distinguishable variations.

Taken together, the comparisons with MetricLearning-scene, ContrastiveScene, and OQ-TSAE$_{\mathrm{raw}}$ suggest that neither pairwise supervision, contrastive supervision, nor architectural factorization alone is sufficient for quotient-consistent representation correctness. The results instead support supervision signals that explicitly reflect scene-relevant distinguishability together with representation structures capable of separating scene-relevant and nuisance-dependent variation.

\paragraph{Per-Family Diagnostic Decomposition.}

Table~\ref{tab:correctness}(b) provides a per-family decomposition of the aggregate diagnostics, allowing geometry, reflectivity, nuisance, and mixed perturbation effects to be examined separately.

The nuisance families reveal the clearest distinction between raw and scene-relevant quotient supervision. For both N-Gain and N-Phase, OQ-TSAE$_{\mathrm{scene}}$ maintains NSR near $0.1$, indicating limited encoding of nuisance-driven measurement variation in the scene latent. In contrast, OQ-TSAE$_{\mathrm{raw}}$ exhibits NSR close to unity, suggesting that nuisance-visible measurement variation remains strongly reflected in latent organization when reliable scene pairs are constructed from raw distinguishability. TuckerAE-Field partially suppresses nuisance variation but remains substantially more sensitive, particularly for phase perturbations. These observations support the interpretation that architectural factorization alone may be insufficient unless the supervision signal is aligned with scene-relevant distinguishability.

The geometry and reflectivity families further indicate that nuisance suppression is not achieved through indiscriminate latent contraction. On the geometry-equivalent family (G-Eq), OQ-TSAE$_{\mathrm{scene}}$ achieves the lowest FDR, consistent with improved contraction of quotient-equivalent scene perturbations. More importantly, on reliably distinguishable reflectivity perturbations (R-Dist), OQ-TSAE$_{\mathrm{scene}}$ maintains zero FMR, whereas TuckerAE-Field continues to collapse a nontrivial fraction of distinguishable scene pairs.

The mixed scene–nuisance perturbation families provide an informative stress test because scene discrimination must be preserved in the presence of nuisance variation. OQ-TSAE$_{\mathrm{scene}}$ maintains zero FMR on both G+Gain and R+Phase families, suggesting that nuisance suppression remains compatible with preserving distinguishable scene structure when scene and nuisance perturbations co-occur.

\subsection{Threshold Sensitivity Analysis}
\label{subsec:threshold_sensitivity}
Two thresholds influence the evaluation pipeline: the supervision-threshold ratio
$\epsilon_{\mathrm{dist}}/\epsilon_{\mathrm{eq}}$,
which controls reliable-pair construction during training,
and the diagnostic threshold $\delta$,
which determines the computation of representation-correctness metrics.
Sensitivity to both quantities is evaluated while keeping all other settings fixed.

\begin{figure}
\centering
\includegraphics[width=\columnwidth]
{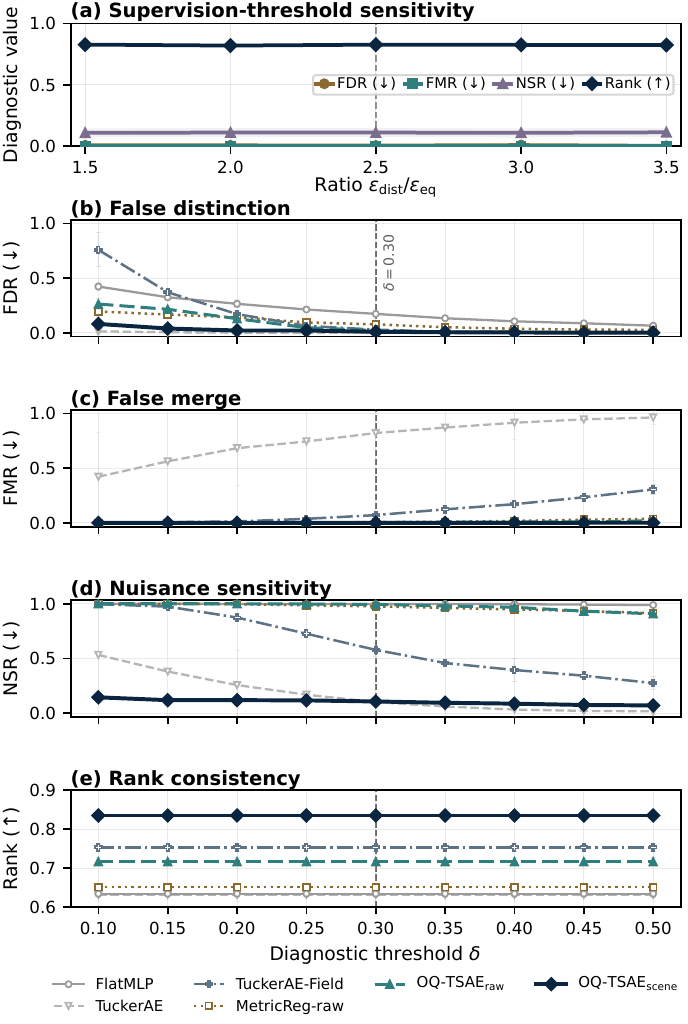}
\caption{
Threshold sensitivity analysis of representation correctness.
(a) Sensitivity to the supervision-threshold ratio
$\epsilon_{\mathrm{dist}}/\epsilon_{\mathrm{eq}}$.
(b)--(e) Sensitivity of false distinction (FDR), false merge (FMR),
nuisance sensitivity (NSR), and rank consistency to the diagnostic
threshold $\delta$.
The dashed lines indicate the operating configurations used throughout
the paper.
}
\label{fig:threshold_sensitivity}
\end{figure}

Figure~\ref{fig:threshold_sensitivity}(a) examines the supervision-threshold ratio.
Across the evaluated range, FDR, FMR, NSR, and rank consistency remain nearly unchanged, indicating that the representation-correctness behavior does not depend on a narrowly selected supervision configuration. Quotient-consistent scene organization therefore remains stable under moderate variation of the reliable-pair construction criterion.

Figure~\ref{fig:threshold_sensitivity}(b)--(e) evaluates diagnostic-threshold sensitivity over
$\delta \in [0.10,0.50]$.
OQ-TSAE$_{\mathrm{scene}}$ maintains low FDR, near-zero FMR, low NSR, and stable rank consistency throughout the evaluated range, indicating that the operating value $\delta=0.30$ lies within a stable diagnostic regime rather than at an isolated favorable point.

The threshold sweep also preserves the qualitative ordering observed in Table~\ref{tab:correctness}(a). Methods based on raw measurement distinguishability remain highly nuisance-sensitive across thresholds, TuckerAE continues to exhibit severe false merging, and TuckerAE-Field remains less nuisance-suppressed and less rank-consistent than OQ-TSAE$_{\mathrm{scene}}$. These results indicate that the principal correctness conclusions are robust to threshold selection and reflect persistent differences in latent organization.

\subsection{Mechanistic Validation of Quotient-Consistent Learning}
\label{subsec:mechanistic_validation}

This experiment investigates the mechanisms underlying quotient-consistent representation organization. Specifically, we evaluate pathway-selective sensitivity, supervision robustness, quotient-relation perturbations, and controlled component ablations. Together, these analyses examine whether scene-relevant variation is preferentially organized within the scene pathway and identify the model components most responsible for representation correctness.

\begin{figure*}
\centering
\includegraphics[width=0.92\textwidth]
{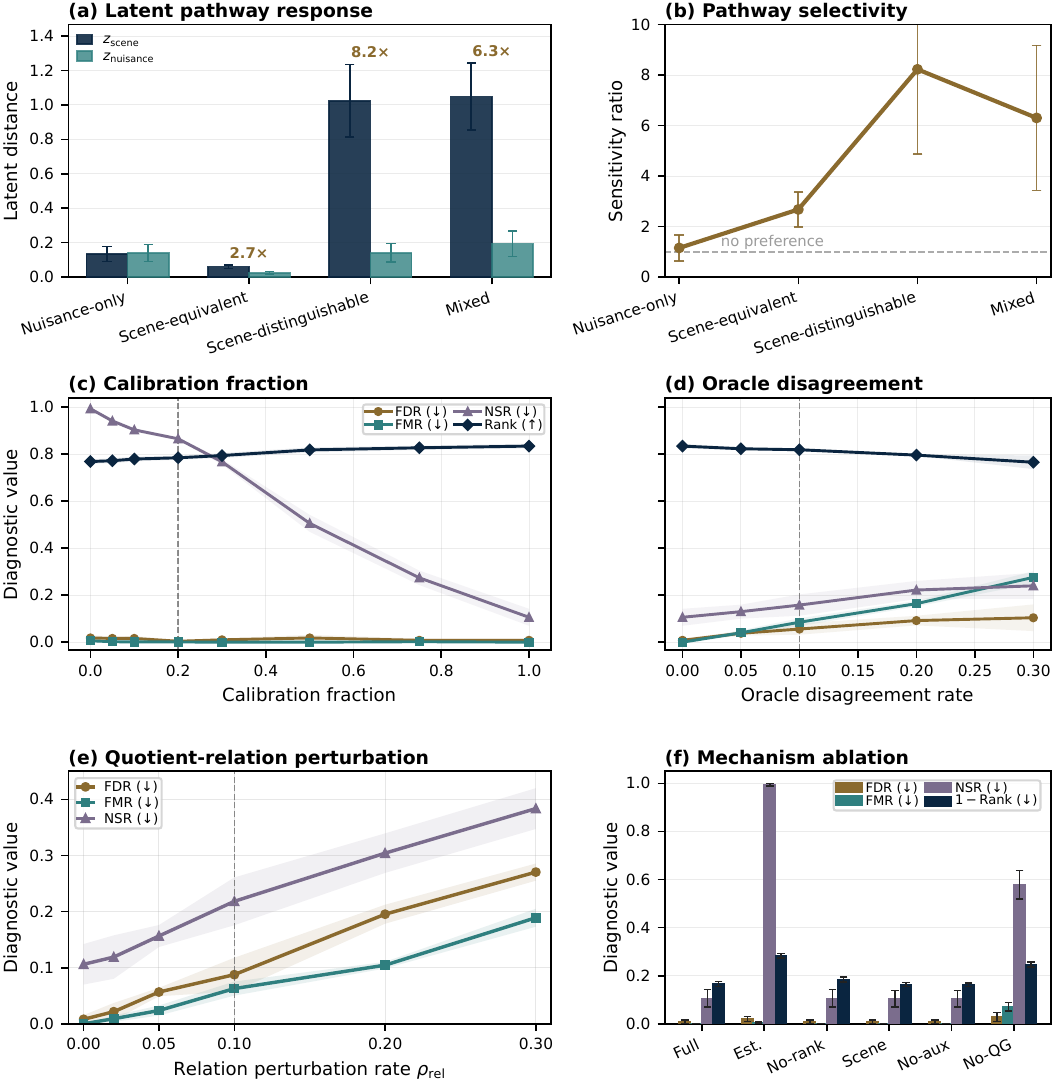}
\caption{
Mechanistic validation of quotient-consistent representation learning.
(a) Latent response of the scene and nuisance pathways across perturbation families.
(b) Scene-to-nuisance pathway selectivity ratio.
(c) Calibration-fraction sensitivity.
(d) Oracle-disagreement sensitivity.
(e) Quotient relation-assignment perturbation sensitivity.
(f) Component ablation analysis.
Dashed lines indicate the operating configurations used throughout the paper.
Rank consistency in (f) is reported as $1-\mathrm{Rank}$, where lower values indicate better latent ordering.
}
\label{fig:mechanistic_validation}
\end{figure*}

Figure~\ref{fig:mechanistic_validation}(a,b) evaluates pathway-selective sensitivity. Nuisance-only perturbations produce comparatively weak scene-pathway responses, whereas scene-distinguishable and mixed perturbations generate substantially larger scene-pathway responses together with elevated scene-to-nuisance selectivity ratios. This behavior is consistent with quotient-consistent organization, where sensing-supported scene distinctions preferentially expand the scene pathway while nuisance variation is largely confined to the nuisance pathway.

Figure~\ref{fig:mechanistic_validation}(c,d) examines robustness to imperfect supervision. Increasing calibration fraction reduces nuisance sensitivity and improves rank consistency, whereas progressively increasing oracle disagreement produces gradual degradation in FDR, FMR, and NSR while leaving latent ordering comparatively stable. These observations indicate that quotient-consistent organization benefits from improved supervision quality but remains relatively robust under moderate supervision mismatch.

Figure~\ref{fig:mechanistic_validation}(e) reports representation diagnostics under progressively increasing quotient relation-assignment perturbation rates. FDR, FMR, and NSR deteriorate in a gradual and largely monotonic manner as perturbation increases, indicating that quotient-consistent organization benefits from reliable quotient relations while remaining relatively robust to moderate relation-assignment errors.

Figure~\ref{fig:mechanistic_validation}(f) reveals a differentiated contribution pattern among model components. Replacing quotient-aligned supervision with raw-distance-derived relation supervision produces the largest degradation, substantially increasing nuisance sensitivity and ordering error. Removing quotient geometry yields the second largest deterioration, particularly in FMR and NSR. By contrast, removing rank consistency primarily affects latent ordering, while auxiliary reconstruction and explicit nuisance-pathway modeling exert comparatively limited influence on the primary correctness diagnostics.

Taken together, these results indicate that representation correctness is driven primarily by supervision aligned to scene-relevant equivalence together with quotient geometry, while pathway separation contributes additional robustness by isolating nuisance variation. The observed behavior is qualitatively consistent with the theoretical perspective developed in Section~\ref{sec:theoretical_analysis}, which identifies quotient-consistent organization as a central determinant of representation correctness.
\subsection{Practical Compatibility and Computational Footprint}
\label{subsec:practical_compatibility}
An important practical question is whether quotient-consistent representation learning remains compatible with scene interpretation and observation recovery, and whether the observed correctness gains arise primarily from increased computational scale. Table~\ref{tab:practical_utility} reports decoding performance, reconstruction quality, and computational footprint under identical implementation settings.

\begin{table*}
\centering
\caption{
Practical utility analysis.
Decoding compatibility, reconstruction performance, and computational footprint comparison.
A dash indicates that the metric is structurally inapplicable or that the corresponding output is not actively trained for the evaluated method.
Forward latency is measured per inference pass and excludes offline pair construction and oracle distinguishability computation performed prior to training.
}
\label{tab:practical_utility}
\setlength{\tabcolsep}{3pt}
\begin{tabular}{lccc|cc|ccc}
\toprule
&
\multicolumn{3}{c|}{\textbf{Decoding}}
&
\multicolumn{2}{c|}{\textbf{Reconstruction}}
&
\multicolumn{3}{c}{\textbf{Footprint}}
\\
\cmidrule(lr){2-4}
\cmidrule(lr){5-6}
\cmidrule(l){7-9}

\textbf{Method}
& \textbf{NMSE}$\downarrow$
& \textbf{F1}$\uparrow$
& \textbf{RMSE}$\downarrow$
& \textbf{Raw}$\downarrow$
& \textbf{Can.}$\downarrow$
& \textbf{Params (M)}
& \textbf{Size (MB)}
& \textbf{Forward (ms)}$\downarrow$
\\

\midrule

FlatMLP
& --
& --
& --
& --
& --
& 16.84
& 64.25
& \underline{0.20$\pm$0.00}
\\

MetricReg-raw
& --
& --
& --
& --
& --
& 16.84
& 64.25
& \textbf{0.19$\pm$0.00}
\\

TuckerAE
& --
& --
& --
& \textbf{0.057}
& \textbf{0.548}
& 2.48
& 9.44
& 1.93$\pm$0.02
\\

TuckerAE-Field
& \textbf{3.194}
& \textbf{0.151}
& \underline{0.204}
& 0.107
& 0.575
& 2.48
& 9.44
& 1.84$\pm$0.01
\\

OQ-TSAE$_{\mathrm{raw}}$
& \underline{3.298}
& \underline{0.147}
& \textbf{0.203}
& \underline{0.099}
& \underline{0.573}
& 2.48
& 9.44
& 1.83$\pm$0.01
\\

\textbf{OQ-TSAE$_{\mathrm{scene}}$}
& 3.300
& 0.142
& 0.242
& 0.124
& 0.581
& 2.48
& 9.44
& 1.84$\pm$0.01
\\

\bottomrule
\end{tabular}
\end{table*}

Although OQ-TSAE$_{\mathrm{scene}}$ is not optimized primarily for decoding or reconstruction, it retains usable decoding and reconstruction performance while substantially improving representation correctness. These results indicate that quotient-consistent scene organization remains compatible with observation recovery and observable scene interpretation.

Table~\ref{tab:practical_utility} also shows that these improvements are not explained by increased model capacity or computational scale. OQ-TSAE$_{\mathrm{scene}}$ shares the same Tucker-structured architecture scale as TuckerAE, TuckerAE-Field, and OQ-TSAE$_{\mathrm{raw}}$, resulting in identical parameter count, model size, and nearly identical inference latency. The representation-correctness gains therefore cannot be attributed to increased model capacity or computational footprint.

\subsection{Scene-Level Distinguishability}
\label{subsec:scene_level_distinguishability}
To examine whether latent organization adapts to sensing-dependent scene distinguishability, we compare scene-relevant distinguishability, latent response, and raw measurement distinguishability across sensing configurations of increasing capability.

\begin{figure*}
\centering
\includegraphics[width=0.98\textwidth]
{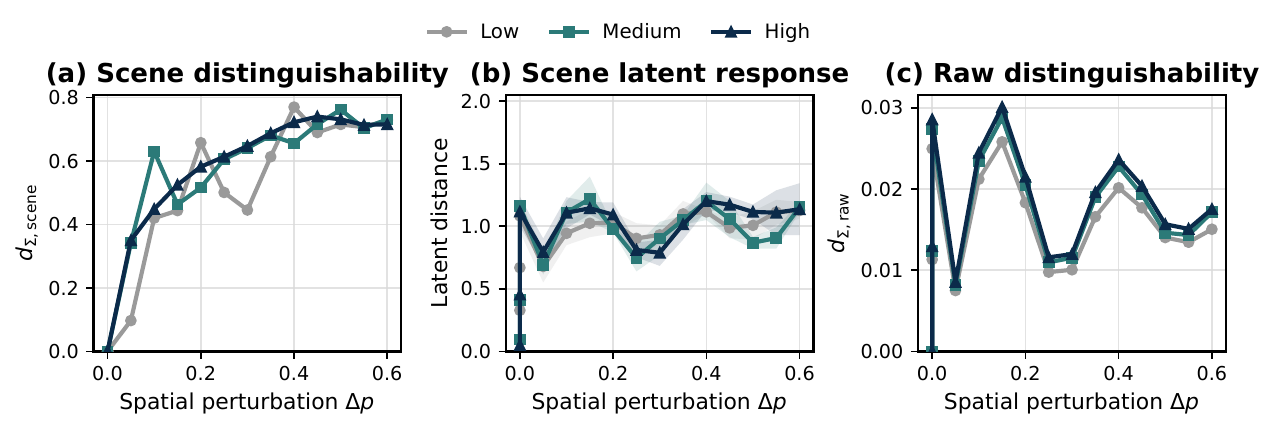}
\caption{
Sensing-dependent scene resolution analysis under sensing configurations of increasing capability.
(a) Scene-relevant distinguishability increases with perturbation magnitude and rises earlier under stronger sensing configurations.
(b) The learned scene latent exhibits capability-dependent separation behavior broadly consistent with scene-relevant distinguishability.
(c) Raw measurement distinguishability computed from uncanonicalized observations reflects additional hardware-dependent variation and serves as a negative control.
}
\label{fig:distinguishability_combined}
\end{figure*}

Figure~\ref{fig:distinguishability_combined}(a) shows scene-relevant distinguishability as a function of perturbation magnitude under sensing configurations of increasing capability: low ($M=8$, $Q=16$, SNR $=10$ dB), medium ($M=16$, $Q=32$, SNR $=20$ dB), and high ($M=32$, $Q=64$, SNR $=30$ dB). Scene-relevant distinguishability generally increases with perturbation magnitude and rises earlier under stronger sensing configurations. Small perturbations that remain weakly distinguishable under low-capability sensing become distinguishable at lower perturbation levels as sensing capability increases. This behavior is consistent with the quotient interpretation that improved sensing capability refines the scene-relevant observation quotient.

Figure~\ref{fig:distinguishability_combined}(b) shows a corresponding trend in latent space. Stronger sensing configurations induce earlier latent separation, whereas weaker configurations keep small perturbations comparatively close together. As perturbation magnitude increases, latent distances gradually approach the limits imposed by the normalized latent geometry. The observed latent organization is therefore broadly consistent with the scene-relevant distinguishability structure rather than with physical displacement alone.

Figure~\ref{fig:distinguishability_combined}(c) reports the distinguishability metric computed from uncanonicalized observations. Unlike $d_{\Sigma,\mathrm{scene}}$, the raw metric also reflects hardware-dependent and calibration-related variation removed through nuisance canonicalization. Consequently, raw distinguishability may attribute nuisance-induced measurement changes to scene differences and therefore provides a negative control for nuisance-sensitive supervision.

Overall, the results suggest that latent organization adapts to sensing-supported scene distinguishability after nuisance canonicalization rather than to physical scene variation or raw measurement variation alone. Changes in sensing capability therefore alter not only observation quality but also the scene distinctions that a quotient-consistent representation is expected to preserve.

\subsection{Real-Radar Utility Validation}
\label{subsec:real_radar_validation}

To complement the controlled sensing benchmark, we evaluate the learned representations on the CARRADA real-radar dataset~\cite{ouaknine2021carrada}, which provides synchronized automotive radar observations and object annotations collected in outdoor driving environments. Range--angle observations are used as encoder inputs. Encoders are frozen after representation learning and evaluated using identical class-balanced logistic probes for frame-level multi-label object-presence prediction. The evaluation uses 7,193 annotated range--angle frames partitioned into 4,319 training, 1,483 validation, and 1,391 test samples.

Because CARRADA does not provide quotient relations, canonical observations, nuisance annotations, or oracle scene-distinguishability targets, evaluation is necessarily based on downstream prediction rather than quotient-oriented correctness diagnostics. Consequently, the results should be interpreted as a downstream utility validation rather than a direct test of quotient correctness. Because quotient-specific supervision is unavailable, only the reconstruction component of OQ-TSAE is activated, yielding a reconstruction-only variant denoted by OQ-TSAE$_{\mathrm{raw}}$.

\begin{table}[t]
\centering
\scriptsize
\caption{
Downstream utility, seed-to-seed variability, and robustness under observation degradation on the CARRADA real-radar dataset.}
\label{tab:real_radar_utility}
\setlength{\tabcolsep}{1pt}

\begin{tabular}{lccccccc}
\toprule

Method
& Macro AP
& Std.$\downarrow$
& Ped.
& Cycl.
& Car
& \makecell{SNR\\Drop$\downarrow$}
& \makecell{Clutter\\Drop$\downarrow$}
\\

\midrule

\makecell[l]{MetricLearning-scene}
& \textbf{0.667}
& 0.044
& 0.905
& \textbf{0.180}
& 0.844
& 0.180
& 0.076
\\

FlatMLP
& \underline{0.661}
& 0.011
& \textbf{0.946}
& 0.120
& 0.953
& \textbf{0.004}
& \textbf{0.004}
\\

\makecell[l]{OQ-TSAE$_{\mathrm{raw}}$}
& 0.657
& 0.010
& 0.911
& 0.116
& \textbf{0.969}
& 0.032
& 0.018
\\

TuckerAE
& 0.656
& 0.012
& 0.891
& 0.095
& 0.946
& \underline{0.030}
& \underline{0.014}
\\

ContrastiveScene
& 0.622
& 0.005
& 0.892
& 0.140
& 0.851
& 0.232
& 0.186
\\

\bottomrule
\end{tabular}
\end{table}

Table~\ref{tab:real_radar_utility} reports downstream utility, seed stability, and robustness on CARRADA. Macro AP is averaged over five independent random seeds, and class-wise AP values are reported for the reference seed used in the real-radar experiments.

The results show that the strongest method depends on the evaluation criterion. MetricLearning-scene achieves the highest Macro AP and cyclist AP, but also has the largest seed variability among the competitive methods and degrades substantially under both stressors. FlatMLP attains the highest pedestrian AP and the smallest degradation under both additive noise and structured clutter. OQ-TSAE$_{\mathrm{raw}}$ achieves the highest car AP and a Macro AP within 0.010 of the strongest method, with variability comparable to FlatMLP and TuckerAE. These results indicate that no method dominates all criteria; instead, the methods occupy different positions in the utility--robustness tradeoff space.

Despite using only reconstruction-based learning and receiving no quotient-specific supervision, OQ-TSAE$_{\mathrm{raw}}$ remains competitive across clean utility, seed stability, and degradation robustness. This supports the use of the real-radar experiment as a downstream utility validation, while the quotient-correctness claims remain grounded in the controlled benchmark.

\begin{figure*}[t]
\centering
\includegraphics[width=\textwidth]
{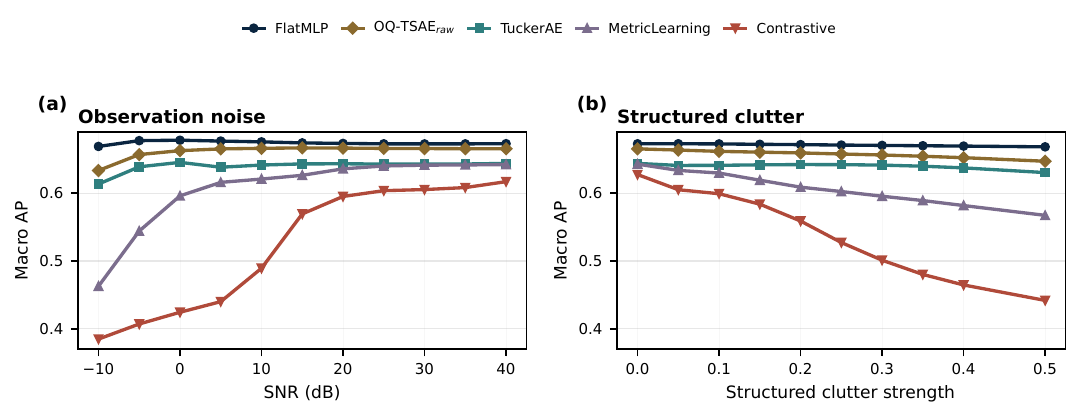}
\caption{
Downstream utility under real-radar observation degradation on the CARRADA dataset.
(a) Additive observation noise across varying test-time SNR levels.
(b) Spatially correlated structured clutter corruption.
Different representation objectives exhibit markedly different robustness profiles under both degradation mechanisms.
}
\label{fig:real_radar_robustness}
\end{figure*}

Figure~\ref{fig:real_radar_robustness} further shows the degradation profiles under two complementary observation stressors. Under additive observation noise (Fig.~\ref{fig:real_radar_robustness}a), performance generally improves as test-time SNR increases, but the rate of degradation differs markedly across objectives. FlatMLP is nearly invariant, OQ-TSAE$_{\mathrm{raw}}$ and TuckerAE degrade only modestly, whereas MetricLearning-scene and ContrastiveScene are substantially more sensitive under low-SNR observations.

Under structured clutter corruption (Fig.~\ref{fig:real_radar_robustness}b), the separation is even clearer. FlatMLP remains nearly invariant, OQ-TSAE$_{\mathrm{raw}}$ and TuckerAE retain relatively stable utility, while MetricLearning-scene and especially ContrastiveScene suffer much larger losses as clutter strength increases. These results indicate that representation objectives influence not only clean downstream utility but also robustness to realistic radar observation degradation.

Taken together, the results indicate that representation objectives influence not only downstream predictive utility but also robustness to observation-quality variation. Although OQ-TSAE$_{\mathrm{raw}}$ does not achieve the highest downstream utility, it consistently combines competitive accuracy, low seed-to-seed variability, and strong robustness under both additive and spatially correlated observation corruption. These findings suggest that the learned tensor representation retains useful task-relevant information under realistic radar observation degradations even without quotient-specific supervision.
\section{Discussion}
\label{sec:discussion}
\subsection{What the Results Establish}
\label{subsec:what_the_results_establish}

The experimental results support the central claim of this work under the controlled sensor-conditioned benchmark considered here: quotient-consistent representation learning can preserve scene-relevant distinctions while suppressing nuisance-driven measurement variation within a unified latent organization. Across the evaluated methods, this behavior is most evident in OQ-TSAE$_{\mathrm{scene}}$, whose diagnostic profile indicates that nuisance suppression need not be achieved through indiscriminate latent collapse. In particular, scene-distinguishable perturbations remain separated while quotient-equivalent variations are more consistently contracted.

The additional ContrastiveScene control and the component ablations further suggest that the observed correctness behavior cannot be attributed solely to a particular architecture or pairwise loss formulation. Instead, quotient-aligned supervision and quotient geometry emerge as the primary contributors under the evaluated setting, while auxiliary reconstruction, rank consistency, and explicit nuisance-pathway modeling appear comparatively secondary with respect to the primary correctness diagnostics.

The calibration and oracle-disagreement analyses indicate that quotient-consistent behavior remains qualitatively stable under moderate supervision mismatch. The complementary real-radar evaluation on CARRADA provides a distinct form of evidence. Because quotient relations, canonical observations, nuisance annotations, and oracle scene-distinguishability targets are unavailable in real radar data, the experiment cannot directly evaluate quotient correctness. Instead, it assesses whether the learned representations retain practical downstream utility, training stability, and robustness when quotient-aligned supervision cannot be instantiated. In this sense, the CARRADA results complement rather than substitute the controlled quotient-correctness evaluation, suggesting that the learned scene--nuisance tensor representation remains competitive in terms of downstream utility, stability, and robustness on a real-radar dataset where quotient-specific supervision is unavailable.

Taken together, these findings should nevertheless be interpreted within the scope of the present study. The experiments provide evidence that quotient-oriented supervision improves representation correctness under controlled sensing conditions, but do not establish universal sensing generalization, identifiable latent disentanglement, or superiority across arbitrary sensing tasks and environments.

\subsection{What the Representation Encodes and What It Filters Out}
\label{subsec:what_the_representation_encodes_and_what_it_filter}

The experimental results help clarify the representational objective induced by quotient-consistent supervision. The scene latent is encouraged to preserve scene-relevant structure that remains reliably distinguishable after nuisance canonicalization. In the present benchmark, this includes observable reflectivity and motion characteristics that remain resolvable under the sensing process. Observable decoding experiments indicate that this latent organization remains compatible with recovering sensor-resolvable scene information.

At the same time, the representation is encouraged to contract variation that is not scene-relevant under the sensing process. Sub-resolution perturbations and nuisance-induced measurement variability removed through canonicalization are treated as quotient-equivalent and therefore need not remain separated in latent space.

The mechanism ablations suggest that quotient-aligned supervision and quotient geometry play a larger role than explicit nuisance modeling under the evaluated setting. Furthermore, the comparison between MetricLearning-scene and ContrastiveScene indicates that pairwise similarity learning alone does not determine which distinctions should be preserved. Even when pair construction is derived from scene-relevant distinguishability, nuisance-visible variation may remain strongly encoded unless additional quotient-oriented structure is imposed.

\subsection{Theoretical Interpretation and Mechanistic Alignment}
\label{subsec:theoretical_interpretation_and_mechanistic_alignme}

The observed diagnostic behavior is broadly consistent with the theoretical perspective developed in Section~\ref{sec:theoretical_analysis}. Across the evaluated methods, lower false distinction, false merge, and nuisance sensitivity rates are associated with latent organizations that more closely preserve sensing-supported scene distinctions while contracting quotient-equivalent variation. The quotient relation-assignment perturbation and component-ablation analyses further suggest that reliable quotient relations and quotient-oriented supervision contribute materially to representation correctness.

Although these empirical observations do not constitute proofs of the theoretical guarantees, they are qualitatively aligned with the central intuition underlying quotient-consistent learning: representations should preserve sensing-supported scene distinctions while contracting observationally equivalent variation.

The framework also admits an information-geometric interpretation under Gaussian observation assumptions, where scene-relevant distinguishability can be viewed as approximating local statistical separation after nuisance canonicalization. However, the quotient formulation itself does not depend on this interpretation, which should be regarded as a theoretical perspective rather than an empirically validated claim of the present study.

Several theoretical and practical questions remain open, including estimation of scene-relevant distinguishability without oracle access, robustness under richer nuisance processes, and extension to real sensing platforms and adaptive sensing strategies.

\subsection{Practical Implications and Limitations}
\label{subsec:practical_implications_and_limitations}

Several limitations should be considered when interpreting the present results.

First, the primary correctness evaluations were conducted in a controlled simulation environment with known forward models and explicitly modeled nuisance factors. Although complementary real-radar evaluations of downstream utility, seed stability, and robustness were performed on the CARRADA dataset, direct evaluation of quotient consistency remains confined to the controlled benchmark setting. Therefore, the CARRADA results should be interpreted as downstream behavioral evidence, whereas the reported correctness results should be regarded as evidence of the proposed mechanism under controlled sensing conditions rather than as a complete characterization of real-world sensing environments.

Second, the current framework assumes access to oracle-derived scene-relevant distinguishability during training. Existing real-world radar datasets do not provide quotient relations, nuisance annotations, or oracle scene-distinguishability targets. Although the calibration analyses suggest that quotient-consistent behavior remains qualitatively stable under moderate supervision mismatch, practical deployment will require reliable procedures for estimating quotient relations from imperfect observations.

Third, the theoretical analysis relies on structured nuisance canonicalization and distinguishability measures derived under Gaussian observation assumptions. The controlled benchmark provides evidence for the proposed quotient-consistency mechanisms under a tractable sensing model, whereas the complementary real-radar evaluation assesses downstream utility only. Whether quotient-consistent guarantees extend to sensing environments involving non-Gaussian interference, heavy-tailed noise, or unmodeled nuisance processes therefore remains an open question.

Fourth, the mechanism ablations indicate that explicit nuisance-pathway modeling appears secondary relative to quotient-aligned supervision and quotient geometry under the evaluated setting. Whether this observation persists under richer sensing regimes warrants further investigation.

Finally, the current framework relies on pairwise supervision derived from scene distinguishability relationships. While effective in the benchmark setting considered here, efficient construction and approximation of such supervision signals may become increasingly important as sensing environments grow in scale and complexity.

Overall, the present results support the view that representation correctness in structured sensing is fundamentally tied to sensing-supported scene distinguishability after nuisance canonicalization. Under this perspective, the central challenge is not merely learning compact latent representations, but determining which scene distinctions remain meaningful under the sensing process and which variations should be contracted. The proposed quotient-consistent formulation provides one principled framework for addressing this question under sensor-conditioned observation processes.

\section{Conclusion}
\label{sec:conclusion}

This paper studied sensor-conditioned representation learning under sensing-dependent distinguishability constraints. We argued that representation correctness cannot be determined solely by reconstruction or downstream prediction performance. Instead, a sensor-conditioned representation should preserve scene-relevant distinctions supported by the sensing process after nuisance canonicalization while contracting sensor-indistinguishable and scene-irrelevant variation.

To formalize this objective, we introduced the \emph{scene-relevant observation quotient} as a representation target defined through sensing-supported distinguishability after nuisance canonicalization. We further developed a quotient-oriented framework connecting scene equivalence, distinguishability, nuisance canonicalization, and representation-correctness diagnostics. As a concrete instantiation, we proposed OQ-TSAE, a Tucker-structured scene--nuisance factorized autoencoding framework for quotient-consistent representation learning.

Experiments on a controlled sensor-conditioned benchmark provide evidence that quotient-consistent supervision improves representation-correctness diagnostics, while sensitivity, perturbation, and ablation analyses highlight the importance of quotient-aligned supervision, reliable quotient relations, and quotient geometry under the evaluated setting. A complementary real-radar evaluation on the CARRADA dataset further suggests that a reconstruction-only OQ-TSAE variant retains competitive downstream utility, low seed-to-seed variability, and favorable robustness under observation degradation, although direct evaluation of quotient correctness in real data remains unavailable.

The present study remains intentionally controlled in scope and should be viewed primarily as a mechanism-level validation of the proposed representation principle, complemented by an initial real-radar assessment of utility, robustness, and stability. These findings suggest that representation learning for structured sensing may benefit from being evaluated not only by reconstruction fidelity or downstream task accuracy, but also by whether the learned latent geometry preserves the scene distinctions that the sensing process can meaningfully support.

\appendix

\setcounter{figure}{0}
\setcounter{table}{0}
\setcounter{equation}{0}

\renewcommand{\thefigure}{\thesection.\arabic{figure}}
\renewcommand{\thetable}{\thesection.\arabic{table}}
\renewcommand{\theequation}{\thesection.\arabic{equation}}

\renewcommand{\theHfigure}{\thesection.\arabic{figure}}
\renewcommand{\theHtable}{\thesection.\arabic{table}}
\renewcommand{\theHequation}{\thesection.\arabic{equation}}
\section{Scene-Pair Taxonomy}
\label{app:scene_pair_taxonomy}

Table~\ref{tab:pair_families} summarizes the eight scene-pair families used throughout the benchmark. Geometry and reflectivity families evaluate whether quotient-equivalent perturbations are contracted while sensing-supported scene distinctions remain separable. Nuisance families evaluate whether measurement-visible but scene-irrelevant variation is suppressed after canonicalization. Mixed scene--nuisance families provide stress tests that simultaneously require scene preservation and nuisance robustness.

\begin{table}
\centering
\caption{Scene-pair taxonomy used for representation-correctness diagnostics.}
\label{tab:pair_families}
\small
\setlength{\tabcolsep}{3pt}
\begin{tabular}{p{1.1cm}p{2.2cm}p{1.2cm}p{0.9cm}p{1.1cm}}
\toprule
\textbf{Fam.} &
\textbf{Scene Change} &
\textbf{Nuisance} &
\textbf{$d_{\Sigma,\mathrm{scene}}$} &
\textbf{Target} \\
\midrule
G-Eq
& Sub-res.\ geometry
& Same
& Low
& Merge \\

G-Dist
& Resolvable geometry
& Same
& High
& Separate \\

R-Eq
& Sub-res.\ reflectivity
& Same
& Low
& Merge \\

R-Dist
& Resolvable reflectivity
& Same
& High
& Separate \\

N-Gain
& None
& Gain
& Low
& Merge \\

N-Phase
& None
& Phase
& Low
& Merge \\
\midrule

G+Gain
& Resolvable geometry
& Gain
& High
& Separate \\

R+Phase
& Resolvable reflectivity
& Phase
& High
& Separate \\
\bottomrule
\end{tabular}
\end{table}

\section{Hyperparameters, Loss Weights, and Core Notation}
\label{app:hyperparameters}

\setcounter{figure}{0}
\setcounter{table}{0}
\setcounter{equation}{0}

For reproducibility,
Tables~\ref{tab:train_hparams},
\ref{tab:loss_weights},
\ref{tab:carrada_protocol},
\ref{tab:field_loss_weights},
and~\ref{tab:notation}
summarize the primary training configuration,
objective-function weights,
real-radar evaluation settings,
internal auxiliary-loss coefficients,
and core notation used throughout the paper.

For robustness evaluation, additive observation noise is generated as
zero-mean Gaussian noise at varying SNR levels, while structured clutter
is generated as a spatially correlated Gaussian random field scaled by the
clutter-strength parameter. These perturbations are applied only during
evaluation and do not affect representation learning. The corresponding
SNR and clutter-strength ranges are summarized in
Table~\ref{tab:carrada_protocol}.

Unless otherwise specified, controlled-benchmark experiments use the configuration listed in Table~\ref{tab:train_hparams}.

\begin{table}
\centering
\caption{Training hyperparameters used in the primary experiments.}
\label{tab:train_hparams}
\footnotesize
\setlength{\tabcolsep}{4pt}
\begin{tabular}{ll}
\toprule
\textbf{Parameter} & \textbf{Value} \\
\midrule

Observation shape $(T,M,R,Q)$
& $(8,16,4,32)$ \\

Tucker ranks $(R_T,R_M,R_R,R_Q)$
& $(4,8,4,8)$ \\

Scene latent dimension
& $64$ \\

Nuisance latent dimension
& $64$ \\

Hidden dimension
& $256$ \\

Learning rate
& $10^{-4}$ \\

Batch size
& $64$ \\

Training epochs
& $100$ \\

Validation fraction
& $0.10$ \\

Mask keep probability
& $0.80$ \\

Checkpoint selection
& Lowest validation loss \\

Random seeds
& $101,202,303,404,505$ \\

Scene equivalence threshold $\epsilon_{\mathrm{eq}}$
& $0.148060$ \\

Scene distinguishability threshold $\epsilon_{\mathrm{dist}}$
& $0.234222$ \\

Diagnostic threshold $\delta$
& $0.30$ \\

Distance margin $m_{\mathrm{dist}}$
& $1.0$ \\

Rank margin $m_{\mathrm{rank}}$
& $0.05$ \\

Scene-distance gap $\epsilon_{\mathrm{gap}}$
& $10^{-4}$ \\

Residual field coefficient $\alpha_{\rho}$
& $0.1$ \\

Occupancy threshold $\tau_{\rho}$
& $0.08$ \\

\bottomrule
\end{tabular}
\end{table}

Complex-valued observations are represented using separate real and imaginary channels in the implementation, while the theoretical formulation treats observations as complex-valued tensors.

\begin{table}
\centering
\caption{Loss weights used in the unified objective of \eqref{eq:full_loss}.}
\label{tab:loss_weights}
\footnotesize
\setlength{\tabcolsep}{4pt}
\begin{tabular}{lc}
\toprule
\textbf{Weight} & \textbf{Value} \\
\midrule
$\lambda_{\mathrm{raw}}$ & $1.0$ \\
$\lambda_{\mathrm{can}}$ & $0.5$ \\
$\lambda_{\mathrm{mask}}$ & $0.3$ \\
$\lambda_{\mathrm{eq}}$ & $1.0$ \\
$\lambda_{\mathrm{dist}}$ & $0.5$ \\
$\lambda_{\mathrm{rank}}$ & $0.3$ \\
$\lambda_{\mathrm{inv}}$ & $1.0$ \\
$\lambda_{\rho}$ & $0.2$ \\
$\lambda_{\mathrm{occ\text{-}logit}}$ & $0.1$ \\
$\lambda_{v}$ & $0.2$ \\
$\lambda_{\mathrm{prim}}$ & $0.1$ \\
$\lambda_{\mathrm{prim\text{-}render}}$ & $0.0$ \\
\bottomrule
\end{tabular}
\end{table}

\begin{table}
\centering
\caption{
Real-radar evaluation protocol on the CARRADA dataset.
}
\label{tab:carrada_protocol}
\footnotesize
\setlength{\tabcolsep}{4pt}
\begin{tabular}{ll}
\toprule
\textbf{Item} & \textbf{Setting} \\
\midrule

Dataset
& CARRADA \\

Input modality
& Range--angle radar observations \\

Encoder protocol
& Frozen representation encoder \\

Probe
& Class-balanced logistic classifier \\

Training seeds
& $101,202,303,404,505$ \\

Train/Validation/Test
& $4319 / 1483 / 1391$ \\

SNR levels (dB)
& $\{-10,0,10,20,30,40\}$ \\

Structured clutter strengths
& $\{0.0,0.1,0.2,0.3,0.4,0.5\}$ \\

Primary metric
& Macro AP \\

\bottomrule
\end{tabular}
\end{table}

Macro AP is computed as the arithmetic mean of pedestrian, cyclist, and car average precision.

\begin{table}
\centering
\caption{Internal coefficients used within the reflectivity-field loss
$\mathcal L_{\rho}$.}
\label{tab:field_loss_weights}
\footnotesize
\setlength{\tabcolsep}{4pt}
\begin{tabular}{lc}
\toprule
\textbf{Coefficient} & \textbf{Value} \\
\midrule
$\lambda_{\mathrm{occ}}^{(\rho)}$
& $1.0$ \\
$\lambda_{\mathrm{bg}}$
& $0.5$ \\
$\lambda_{\mathrm{phase}}$
& $0.1$ \\
$\lambda_{\mathrm{sp}}$
& $0.25$ \\
\bottomrule
\end{tabular}
\end{table}

\begin{table}
\centering
\caption{Core notation used throughout the paper.}
\label{tab:notation}
\footnotesize
\setlength{\tabcolsep}{3pt}
\renewcommand{\arraystretch}{1.05}

\begin{tabular}{p{0.28\columnwidth} p{0.62\columnwidth}}
\toprule
\textbf{Symbol} & \textbf{Meaning} \\
\midrule

$s_t,\eta_t$
& Scene state and nuisance state. \\

$a^{\mathrm{act}}=(m,r)$
& Sensing action (aperture and response mode). \\

$\mathbf Y,\widetilde{\mathbf Y}$
& Observation tensor and canonicalized observation tensor. \\

$\mathcal Q_{\mathrm{scene}}^{\mathcal U,\mathcal C}$
& Scene-relevant observation quotient. \\

$d_{\Sigma,\mathrm{scene}},d_{\Sigma,\mathrm{raw}}$
& Scene-relevant and raw observation distances. \\

$\mathcal P_{\mathrm{eq}}^{\mathrm{scene}},
\mathcal P_{\mathrm{dist}}^{\mathrm{scene}}$
& Reliable equivalent and distinguishable scene-pair sets. \\

$\mathcal P_{\mathrm{amb}}^{\mathrm{scene}}$
& Ambiguous scene-pair set. \\

$\delta$
& Diagnostic distance threshold used in FDR/FMR/NSR evaluation. \\

$\epsilon_{\mathrm{eq}},
\epsilon_{\mathrm{dist}}$
& Supervision thresholds. \\

$\mathbf G^{\mathrm{scene}},
\mathbf G^{\mathrm{nui}}$
& Scene and nuisance Tucker cores. \\

$\bar z^{\mathrm{scene}}$
& Normalized scene latent representation. \\

$\sim_{\mathcal U,\mathcal C}^{\mathrm{scene}}$
& Scene equivalence relation induced by the observation quotient. \\

$z^{\mathrm{scene}},
z^{\mathrm{nui}}$
& Scene and nuisance latent representations. \\

$\mathcal L_{\mathrm{eq}},
\mathcal L_{\mathrm{dist}},
\mathcal L_{\mathrm{rank}}$
& Quotient-consistency losses. \\

$\operatorname{Rank}$
& Rank-consistency diagnostic. \\

\bottomrule
\end{tabular}
\end{table}

\section*{CRediT Authorship Contribution Statement}

Yan Jiao: Conceptualization, Methodology, Formal analysis, Investigation, Software, Validation, Visualization, Writing -- original draft, Writing -- review and editing.

Pin-Han Ho: Supervision, Project administration, Methodology, Writing -- review and editing.

Limei Peng: Supervision, Writing -- review and editing.

\section*{Data Availability}

The controlled-benchmark data used in this study were generated procedurally from the sensor-conditioned simulation benchmark described in the manuscript. The complementary real-radar evaluation uses the publicly available CARRADA dataset~\cite{ouaknine2021carrada}. The source code, configuration files, and benchmark-generation scripts required to reproduce the controlled-benchmark experiments will be made publicly available upon acceptance. During review, they are available from the corresponding author upon reasonable request.

\section*{Declaration of Competing Interest}

The authors declare that they have no known competing financial interests or personal relationships that could have appeared to influence the work reported in this paper.

\section*{Declaration of Generative AI and AI-assisted Technologies in the Manuscript Preparation Process}

During the preparation of this manuscript, the authors used generative AI and AI-assisted tools for language editing and readability improvement. After using these tools, the authors reviewed and edited the content as needed and take full responsibility for the content of the manuscript.

\bibliographystyle{cas-model2-names}
\bibliography{references}
\end{document}